%% file: main.tex
\documentclass[twoside,11pt]{article}

\usepackage{blindtext}

\usepackage[preprint,abbrvbib]{jmlr2e}

\usepackage[utf8]{inputenc} 
\usepackage[T1]{fontenc}    
\usepackage{hyperref}       
\usepackage{url}            
\usepackage{booktabs}       
\usepackage{amsfonts}       
\usepackage{nicefrac}       

\usepackage[dvipsnames]{xcolor}
\usepackage{wrapfig}
\usepackage{braket}
\usepackage{xspace}
\usepackage{amsmath}
\usepackage{mathtools}
\usepackage{algorithm}
\usepackage[noend]{algorithmic}
\usepackage[capitalize,noabbrev]{cleveref}
\usepackage{enumitem}
\usepackage{subcaption}
\usepackage{multicol}

\usepackage{lastpage}
\jmlrheading{vol}{2025}{1-\pageref{LastPage}}{12/25}{date}{id}{Adit Vishnu, Abhay Shastry, Dhruva Kashyap, and Chiranjib Bhattacharyya}

\ShortHeadings{Density Operator Expectation Maximization}{Vishnu, Shastry, Kashyap, and Bhattacharyya}
\firstpageno{1}

\DeclareMathOperator{\argmax}{argmax}
\DeclareMathOperator{\argmin}{argmin}
\newlist{questions}{enumerate}{1}
\setlist[questions,1]{label=Q\arabic*, ref=Q\arabic*, leftmargin=*}
\crefname{questionsi}{Question}{Questions}

\newlist{contributions}{enumerate}{1}
\setlist[contributions,1]{label=C\arabic*, ref=C\arabic*, leftmargin=*}
\crefname{contributionsi}{Contribution}{Contributions}

\begin{document}

\title{Density Operator Expectation Maximization}

\author{\name Adit Vishnu \email aditvishnu@iisc.ac.in \\
        \name Abhay Shastry \email abhayshastry@fsid-iisc.in\\
        \name Dhruva Kashyap \email kdhruva@iisc.ac.in\\
        \name Chiranjib Bhattacharyya \email chiru@iisc.ac.in\\
       \addr Department of Computer Science and Automation\\
       Indian Institute of Science\\
       Bangalore, KA 560012, India}

\editor{My editor}

\input{symbols}

\maketitle
\begin{abstract}
\input{Sections/0_abstract}
\end{abstract}


\input{Sections/1_intro}

\input{Sections/2_ml}
\input{Sections/2_quantum}
\input{Sections/3_rw}
\input{Sections/4_dolvm}
\input{Sections/5_doem}
\input{Sections/6_classical}
\input{Sections/7_exp}
\input{Sections/8_conclusion}

\newpage
\appendix
\input{Appendix/1_analysis}
\input{Appendix/2_hardware}
\vskip 0.2in
\bibliography{refs}

\end{document}

%% file: symbols.tex
\def\t{^\intercal}
\def\h{^\dagger}
\newcommand{\iter}[1]{^{(#1)}}

\def\vis{{\scriptscriptstyle\mathrm{V}}}
\def\hid{{\scriptscriptstyle\mathrm{L}}}

\def\data{\mathcal{D}}
\def\cL{\mathcal{L}}
\def\eq{\!=\!}
\def\DBM{\ensuremath{\tt DBM}\xspace}
\def\doms{\textbf{DOM}s\xspace}
\def\dolvm{{DO-LVM}\xspace}
\def\dolvms{{DO-LVMs}\xspace}
\def\qlvms{{CQ-LVMs}\xspace}
\def\qlvm{{CQ-LVM}\xspace}

\def\qbm{\ensuremath{\tt QBM}\xspace}

\def\R{\mathbb{R}}
\def\hil{\mathcal{H}}
\def\tr{\mathrm{Tr}}
\def\map{\mathcal T}
\def\den{\mathcal{P}}
\def\L{\mathcal{L}}
\def\O{\mathcal{O}}
\def\B{\mathcal{B}}
\def\N{\mathcal{N}}
\def\S{\mathcal{S}}
\def\H{\mathrm{H}}
\def\id{\mathrm{I}}
\def\du{\mathrm{D_U}}
\def\svn{\mathrm{S_{VN}}}
\def\Lu{\L_\mathrm{U}}
\def\Lp{\L_\mathrm{P}}
\def\ker{\mathrm{ker}}

\def\rhoh{\rho_{\hid}}
\def\etah{\widehat{\eta}}
\def\rhov{\rho_{\vis}}
\def\rhovh{\rho_{\vis\hid}}
\def\rhot{\widetilde{\rho}}
\def\etav{\eta_{\vis}}
\def\etat{\widetilde{\eta}}
\def\etah{\eta_{\hid}}
\def\sigv{\widetilde{\sigma}}
\def\sigh{\sigma_{\hid}}
\def\etae{\eta_{\epsilon}}
\def\etate{\tilde{\eta}_{\epsilon}}

\def\sz{\sigma\iter{z}}
\def\sx{\sigma\iter{x}}

\def\v{\mathbf v}
\def\u{\mathbf u}
\def\e{\mathbf e}

\def\QELBO{\mathrm{QELBO}}
\def\doem{{DO-EM}\xspace}
\def\MRE{{\tt MRE}\xspace}
\def\PRM{{\tt PRM}\xspace}
\def\CAO{{\tt CAO}\xspace}
\def\conds{{\nameref{defn:conds}}\xspace}
\def\qidbm{\ensuremath{\tt QiDBM}\xspace}
\def\qrbm{\ensuremath{\tt QRBM}\xspace}
\def\qbms{\ensuremath{\tt QBMs}\xspace}

\def\class{{CQ}\xspace}
\def\RBM{\ensuremath{\tt RBM}\xspace}

\newcommand{\kl}[2]{\mathrm{D_{KL}}\left(#1,#2\right)}

\def\x{\mathrm x}
\def\X{\mathrm X}
\def\z{\mathrm z}
\def\Z{\mathrm Z}

%% file: Sections/0_abstract.tex
Machine learning with density operators, the mathematical foundation of quantum mechanics, is gaining prominence with rapid advances in quantum computing. Generative models based on density operators cannot yet handle tasks that are routinely handled by probabilistic models. The progress of latent variable models, a broad and influential class of probabilistic unsupervised models, was driven by the Expectation-Maximization framework. Deriving such a framework for density operators is challenging due to the non-commutativity of operators. To tackle this challenge, an inequality arising from the monotonicity of relative entropy is demonstrated to serve as an evidence lower bound for density operators. A minorant-maximization perspective on this bound leads to Density Operator Expectation Maximization (DO-EM), a general framework for training latent variable models defined through density operators. Through an information-geometric argument, the Expectation step in DO-EM is shown to be the Petz recovery map. The DO-EM algorithm is applied to Quantum Restricted Boltzmann Machines, adapting Contrastive Divergence to approximate the Maximization step gradient. Quantum interleaved Deep Boltzmann Machines and Quantum Gaussian-Bernoulli Restricted Boltzmann Machines, new models introduced in this work, outperform their probabilistic counterparts on generative tasks when trained with similar computational resources and identical hyperparameters.

%% file: Sections/1_intro.tex
\section{Introduction}\label{sec:intro}
Recent advances in quantum hardware and hybrid classical-quantum algorithms have sparked growing interest in machine learning models capable of operating within quantum regimes \citep{preskill2018}. Quantum machine learning promises to harness quantum mechanics to enhance learning, either by achieving computational speedups on quantum devices \citep{kerenidis2019,kerenidis2020,miyahara2020,kim2023} or by modeling richer data structures through quantum representations \citep{schuld2015,biamonte2017,amin2018}. The latter approach employs density operators---positive semi-definite, unit-trace operators on Hilbert spaces---that generalize the notion of probability distributions and underpin the mathematical foundations of quantum mechanics \citep{vonneumann1955,gleason1957}. While significant progress has been made in the use of density operators in supervised learning, there is relatively less progress in the unsupervised setting \citep{gujju2024}.

Latent variable models (LVMs) are a cornerstone of unsupervised learning, offering a principled approach to modeling complex data distributions through the introduction of unobserved or hidden variables \citep{bishop2006}. These models facilitate the discovery of underlying structure in data and serve as the foundation for a wide range of tasks, including generative modeling, clustering, and dimensionality reduction \citep{jordan1999,ghahramani2001}. Classical examples such as Factor Analysis \citep{basilevsky2009statistical}, Independent Component Analysis \citep{ICA}, Hidden Markov Models \citep{baum1967}, and Boltzmann Machines \citep{ackley1985} exemplify the power of latent variable frameworks in capturing dependencies and variability in observed data. In recent years, LVMs have formed the conceptual backbone of deep generative models including Variational Autoencoders \citep{kingma2014} and Diffusion-based models \citep{ho2020}. The evidence lower bound (ELBO) and the Expectation-Maximization (EM) algorithm \citep{baum1967,dempster1977} has been instrumental in deriving procedures for learning latent variables models. These techniques remain the standard approach for learning latent variable models as they offer a principled and tractable means of handling incomplete data.

Density operators were first studied as latent variable models by \citet{warmuth2005}. There is renewed interest in density operator latent variable models (\dolvm) due to their relevance in quantum computing, yet development remains in its early stages, with fundamental questions concerning expressivity, inference, and learning still largely unresolved \citep{amin2018,kieferova2017,kappen2020,coopmans2024}. 
Most existing learning algorithms for \dolvms fail to scale beyond 12-bit binary strings, restricting evaluation to toy datasets and preventing any meaningful assessment of their modeling power on real-world data. EM-based algorithms can provide a simpler alternative to existing learning algorithms for \dolvms which directly maximizes the likelihood.  However, deriving analogous algorithms in the density-operator setting is challenging for several reasons, primarily because operator-theoretic inequalities---such as the operator Jensen inequality---cannot be directly applied to obtain an ELBO-style bound for \dolvms. A precise characterization of models compatible with such bounds, along with their computational behavior, remains an important direction for investigation.

\subsection{Latent Variable Models}
Latent variable models (LVMs) \citep{bishop2006} model the joint probability distribution of a random variable $\X=(\X_1,\dots,\X_{d_\vis})$ with parameters \(\theta\) as 
\[
    \Pr(\X\mathord{=}\,{\x}\mid\theta) = \sum_{\z} \Pr(X \mathord{=}\,{\mathrm x} , \Z \mathord{=}\,{\mathrm z}\mid\theta)
\]
where $\Z=(\Z_1,\dots,\Z_{d_\hid})$ are unobserved \emph{hidden variables}. Maximum likelihood-based methods estimate the parameters of an LVM by maximizing the log-likelihood 
\begin{equation*}
    \L(\data,\theta) = \frac{1}{N}\sum_{i=1}^N \ell_i( \theta)\ \text{where}\ \ell_i(\theta) = \log \Pr(\X\mathord{=}\, {\mathrm x}^{(i)}\mid \theta)
\end{equation*}
from a data set $\data=\{\x\iter{1},\dots,\x\iter{N}\}$.
However, this maximization problem is intractable in most cases and gradient-based algorithms are difficult to implement because of unwieldy computations in $\ell_i(\theta)$ due to the marginalization within the logarithm.

Boltzmann Machines (BM) are a class of LVMs inspired by the Ising model in statistical physics \citep{ackley1985}. BM variants such as Restricted Boltzmann Machines \citep[RBM,][]{smolensky1986}, Gaussian-Bernoulli RBMs \citep[GRBM,][]{grbm_harmonium}, and Deep Boltzmann Machines \citep[DBM,][]{salakhutdinov2009} could be trained on image data sets, learning meaningful latent representations and generating coherent samples. RBMs and GRBMs were successfully applied to natural images, digits, and simple textures, while DBMs extended this capability to deeper hierarchical representations---an early precursor to modern generative models.
\subsection{EM algorithm}
We recall the EM algorithm used to train probabilistic latent variable models \citep{bishop2006}.The evidence lower bound (ELBO),
\begin{align}\label{ELBO}\tag{ELBO}
    \ell_i(\theta)
    \geq \sum_{\z} q_i(\z)\log \frac{p_\theta(\x^{(i)},\z)}{q_i(\z)},
\end{align}
is a well known lower bound on the log-likelihood for any choice of a \emph{variational distribution} $q_i(\z)$, with equality when $q_i(\z) = p(\z \mid \x^{(i)})$. The Expectation-Maximization algorithm \citep{baum1967,dempster1977} arises by choosing 
$q_i(\z) = p_{\theta^{(\mathrm{old})}}(\z \mid \x^{(i)})$, 
which maximizes the ELBO with respect to the variational distribution $q_i$ (\ref{eq:Estep}), followed by maximizing it with respect to model parameters $\theta$ (\ref{eq:Mstep}):
\begin{align*}
Q_i(\theta \mid \theta^{(\mathrm{old})})
&= \sum_{\z} p_{\theta^{(\mathrm{old})}}(\z \mid \x^{(i)})
\log \left( \frac{p_\theta(\x^{(i)},\z)}{p_{\theta^{(\mathrm{old})}}(\z \mid \x^{(i)})} \right),\tag{E step}\label{eq:Estep} \\
\theta^{(\mathrm{new})} 
&= \arg\max_{\theta} \frac{1}{N}\sum_{i=1}^N Q_i(\theta \mid \theta^{(\mathrm{old})}). \tag{M step}\label{eq:Mstep}
\end{align*}
Here, $Q_i$ is a \emph{minorant} of the log-likelihood \citep{deleeuw1994,lange2000}, satisfying
\begin{align}\label{eq:guarantee}
\ell_i(\theta) \ge Q_i(\theta \mid \theta^{(k)})\ \text{for all } \theta\ \text{and}\ 
\ell_i(\theta^{(\mathrm{old})}) = Q_i(\theta^{(\mathrm{old})} \mid \theta^{(\mathrm{old})}),
\end{align}
which ensures 
$\mathcal L(\mathcal D, \theta^{(\mathrm{new})}) 
\ge \mathcal L(\mathcal D, \theta^{(\mathrm{old})})$. 

The EM algorithm applied to different latent variable models has led to several unsupervised learning algorithms such as k-means clustering and the Baum-Welch algorithm for Hidden Markov Models \citep{bishop2006}. Its formulation---optimizing a surrogate lower bound when direct likelihood maximization is intractable---motivated variational inference, where complex problems are turned into simpler ones by decoupling the degrees of freedom and introducing additional variational parameters that can be optimized to approximate otherwise intractable distributions \citep{jordan1999}. The goal of this paper is to derive an EM algorithm for density operators.
\subsection{Problem Statement and Contributions}
 Quantum Boltzmann Machines (\qbm) are density operator extensions of BMs based on the transverse field Ising model \citep{amin2018}. Unlike their probabilistic counterparts, \qbms, have yet to demonstrate comparable success in generative modeling. Their formulation remains largely theoretical and on toy data sets. Practical training methods for \qbms on complex data remain an open challenge. To bridge this gap, we ask the following questions:
\begin{enumerate}
    \item \textbf{Operator theoretic challenges:} The non-commutativity of operators restricts the use of techniques from probabilistic latent variable models when deriving training algorithms for \dolvms. 
    Is it possible to derive computationally efficient iterative schemes in the spirit of EM for \dolvms with guarantees similar to \cref{eq:guarantee}?
    \item \textbf{Classical data sets:} Since traditional machine learning datasets are not quantum mechanical in nature, can \dolvms be adapted to learn from such data? Specifically, could classical-quantum states used in quantum error correction, where classical information guides the recovery of quantum states, provide a framework for addressing such machine learning tasks?
    \item \textbf{Scaling to standard ML data sets:} Current methods cannot train \dolvms on standard image data sets such as MNIST due to algorithmic limitations. Are there variants of \dolvms that can handle such data? Specifically, can varaints of the classical Boltzmann Machine, such as GRBMs and DBMs, be extended to \qbms and trained on image data sets? If so, is there a modeling advantage to using density operator models over classical probabilistic models for learning such data?
\end{enumerate}

In this paper, we answer these questions by making the following contributions.
\begin{enumerate}
    \item \textbf{EM algorithm for Density Operators}: The Density Operator Expectation Maximization (\doem) algorithm is derived in \cref{sec:doem}. Using the information geometric interpretation of the EM algorithm, the E step is shown to be the well known Petz Recovery Map. Furthermore, \cref{thm:ascent} guarantees log-likelihood ascent across iterations of \doem under mild assumptions, while still accommodating a rich class of models. The two steps of the EM algorithm are computationally more efficient than existing techniques for computing the gradient of the log-likelihood in \dolvms. 
    \item \textbf{Classical-Quantum Models:} The \doem algorithm is specialized to standard ML datasets in \cref{sec:classic}. We introduce a new class of \dolvms, termed Classical-Quantum LVMs (\qlvm), which feature classical visible states and quantum latent states. \qlvms preserve likelihood ascent under the \doem algorithm and eliminate the dependence of \doem on the size of the visible sample space. Quantum counterparts of DBMs and GRBMs are provided as concrete examples of \qlvms.
    \item \textbf{Generative capability:} The first empirical evidence of a modeling advantage for \dolvms trained on classical hardware using image datasets is presented in \cref{sec:exp}. Under comparable computational budgets and using identical hyperparameters—without any additional tuning for the quantum variants—quantum DBMs trained on MNIST achieve a 40-60\% reduction in Fréchet Inception Distance compared to Deep Boltzmann Machines that approximate the log-likelihood gradient using contrastive divergence. Similarly, quantum GRBMs trained on Fashion-MNIST and CelebA-32 achieve a 20-30\% reduction in Fréchet Inception Distance relative to Gaussian-Bernoulli Restricted Boltzmann Machines trained via contrastive divergence.
\end{enumerate}
\subsection{Organization}
\cref{sec:lvm} provides the necessary background on training latent variable models. \cref{sec:qi} introduces key properties of density operators along with useful trace inequalities. \cref{sec:qbm} reviews prior efforts to train \qbms and outlines the limitations of existing approaches. \cref{sec:dolvm} develops a formal treatment of \dolvms and establishes the quantum evidence lower bound. \cref{sec:doem} presents the \doem algorithm and proves its central properties. \cref{sec:classic} then specializes \dolvms and \cref{alg:doem} to classical data sets and introduces new models consistent with this specialization. Finally, \cref{sec:exp} provides empirical validation for the claims developed in the preceding sections.

%% file: Sections/2_ml.tex
\section{Training Latent Variable Models}\label{sec:lvm}
This section provides the relevant background on training probabilistic LVMs. 

\subsection{EM through the lens of the Data Processing Inequality}
The Expectation-Maximization (EM) algorithm \citep{baum1967, dempster1977}, a two-step iterative method for training latent-variable models, has long been the workhorse of unsupervised learning. In this section, we derive the EM algorithm using the data processing inequality to prepare for the subsequent quantum formulation.

Let \(q(\X)\) be the density of the data distribution to be learned by the LVM \(p_\theta(\X,\Z)\). 
Let \(q(\X,\Z)\) be any joint distribution that marginalizes to $q(\X)=\sum_{\z}q(\X,\Z)$. From the data processing inequality \citep{shannon1948} applied to marginalization, we know
\begin{equation}\tag{DPI}\label{eq:DPI}
\kl{q(\X,\Z)}{p_\theta(\X,\Z)}\geq \kl{q(\X)}{p_\theta(\X)}.
\end{equation}
From the chain rule of KL Divergence \citep{cover2006}, 
\begin{equation*}
\kl{q(\X,\Z)}{p_\theta(\X,\Z)} = \kl{q(\X)}{p_\theta(\X)} + \sum_{\mathrm x} q(\X\eq\x) \kl{q(\Z|\X\eq\x)}{p_\theta(\Z|\X\eq\x)}
\end{equation*}
we know that the data processing inequality is saturated when \(p_\theta(\Z|\x) = q(\Z|\X\eq\x)\) for all $\x$, that is, the model and data posteriors are the same. 

Expanding the definition of KL divergence in the data processing inequality results in
\begin{equation*}
\sum_\x q(\X\eq\x)\log{p_\theta(\X\eq\x)}\geq -\kl{q(\X\eq\x,\Z\eq\z)}{p_\theta(\X\eq\x,\Z\eq\z)} - \mathrm{H_S}(q(\Z))
\end{equation*}
where \(\mathrm{H_S}(q(\Z)) = -\sum_\x q(\X\eq\x) \log{q(\X\eq\x)}\) is the Shannon entropy of the data distribution. When the data distribution is an empirical distribution over a finite set of points, \(q(\X\eq\x) = \nicefrac{1}{N}\sum_{i=1}^N \delta(\x\iter{i})\) where \(\delta(\x\iter{i})\) is the Dirac measure at point \(\x\iter{i}\), we may simplify the chain rule as,
\[
\frac{1}{N}\sum_{i=1}^N \log{p_\theta(\X\eq\x\iter{i})} \geq \frac{1}{N}\sum_{i=1}^N \sum_{\z}q(\Z\eq\z|\X\eq\x\iter{i}) \log\frac{p_\theta(\X\eq\x\iter{i},\Z\eq\z)}{q(\Z\eq\z|\X\eq\x\iter{i})}.
\]
We observe that this inequality is the averaged \ref{ELBO} for all data points. Alternatively maximizing the bound with respect to $q$ and $\theta$ results in the EM algorithm described in \cref{sec:intro}.

\subsubsection{E Step as Information Projection}
The Expectation step of the EM algorithm may be seen as a solution to an information projection problem \citep{amari1995,hino2024}. 

\begin{definition}[Information Projection]
    The information projection of a probability distribution $p$ onto a set of probability distributions $\mathcal S$ is 
    \begin{align*}
        r^*=\underset{r\in \mathcal S}{\argmin}\,\kl{r}{p}.
    \end{align*}
\end{definition}
The E step of the EM algorithm is an information projection of the model $p_{\theta^{(\textrm{old})}}(\X,\Z)$  onto the set of joint distributions that marginalize to the empirical data distribution $q(\X)$. Formally, the solution to the information projection problem 
    \begin{align*}
    q^*(\X,\Z)=\underset{\sum_{\z}r(\X,\z)=q(\X)}{\argmin}\kl{r(\X,\Z)}{p_{\theta}(\X,\Z)},
    \end{align*}
is given by $q^*(\X,\Z)=q(\X)p_{\theta}(\Z|\X)$ \citep{cover2006}. It is easy to see that $q*$ saturates the \ref{eq:DPI} and gives us the \cref{eq:Estep}.

\subsection{Boltzmann Machines}

This section discusses variants of Boltzmann Machines that played a key role in the development of modern deep learning \citep{rbmclassification,rbmcollab,rbmdim}, and the contrastive divergence algorithm for learning them.

\subsubsection{Boltzmann Machines and Variants} Boltzmann Machines (BM) are stochastic neural networks that define a probability distribution over binary vectors based on the Ising model in statistical physics \citep{ackley1985}. A Boltzmann Machine (BM) consists of a visible layer $\mathrm{x} \in \{0,1\}^m$ and a hidden layer $\mathrm{z} \in \{0,1\}^n$.
It is parameterized by
\[
\theta = \{\mathbf{W}, \mathbf{W}^{(\vis)}, \mathbf{W}^{(\hid)}, \mathbf{a}, \mathbf{b}\},
\]
where $\mathbf{W} \in \mathbb{R}^{m \times n}$ denotes the interaction weights between visible and hidden layers,
$\mathbf{W}^{(\vis)} \in \mathbb{R}^{m \times m}$ and $\mathbf{W}^{(\hid)} \in \mathbb{R}^{n \times n}$ are symmetric, hollow within-layer weight matrices,
$\mathbf{a} \in \mathbb{R}^m$ is the visible bias vector, and
$\mathbf{b} \in \mathbb{R}^n$ is the hidden bias vector. The joint distribution is defined through an energy function $E_\theta$ as
\begin{align*}
     E_\theta(\mathrm{x},\mathrm{z})
     = -\mathrm{x}^\top \mathbf{W} \mathrm{z}
      - \frac{1}{2}\mathrm{x}^\top \mathbf{W}^{(\vis)} \mathrm{x}
      - \frac{1}{2}\mathrm{z}^\top \mathbf{W}^{(\hid)} \mathrm{z}
      - \mathbf{a}^\top \mathrm{x}
      - \mathbf{b}^\top \mathrm{z},\\
      \Pr(\mathrm{x},\mathrm{z}\mid\theta)
    = \frac{1}{\mathcal Z(\theta)} \exp\!\left( E_\theta(\mathrm{x},\mathrm{z}) \right)\ \text{and}\ 
    \quad\mathcal Z(\theta)
    = \sum_{\mathrm{x}',\mathrm{z}'} 
      \exp\!\left( E_\theta(\mathrm{x}',\mathrm{z}') \right).
\end{align*}

\emph{Restricted Boltzmann Machines} (RBM) were introduced with no intra-layer connections due to the intractability of learning in fully connected BMs \citep{smolensky1986,hinton2002,carreira-perpinan2005}. The energy function of an RBM  simplifies to
\begin{align*}\tag{{\tt RBM}}\label{eq:RBM}
    E_\theta(\mathrm{x},\mathrm{z})
    = -\mathrm{x}^\top \mathbf{W} \mathrm{z}
      - \mathbf{a}^\top \mathrm{x}
      - \mathbf{b}^\top \mathrm{z},
\end{align*}
enabling efficient Gibbs sampling using the Contrastive Divergence algorithm discussed in the next section. 

\emph{Gaussian-Bernoulli RBMs} (GRBM) extend RBMs to model continuous visible units \(\x\in\mathbb{R}^m\) and discrete hidden units \(\z\in\{0,1\}^n\) \citep{grbm_harmonium,liao2023gaussianbernoulli}. It is additionally parameterized by a variance term \(\mathbf{s}\in \mathbb{R}^m_{>0}\). The energy and partition functions are given by
\begin{equation*}\tag{{\tt GRBM}}\label{eq:GRBM}
\begin{aligned}
    E_{\theta}(\x,\z) &= -\frac{1}{2}\sum_{i=1}^m\frac{(\x_i-\mathbf a_i)^2}{s_i} - \sum_{i=1}^m\sum_{j=1}^n\mathbf W_{ij}\frac{\x_i}{s_i}\z_j-\sum_{j=1}^n\mathbf{b}_j\z_j,\\
    \mathcal Z(\theta) &= \int_{-\infty}^{\infty} \sum_{\z}\exp(E_\theta(\x,\z))\mathrm d\x.
\end{aligned}
\end{equation*}

\emph{Deep Boltzmann Machines} (DBM) stack RBMs using undirected connections and allow for joint training of all layers \citep{salakhutdinov2009}. The energy function of a DBM with \(L\) layers, \(\Pr(\x, \z_{[1]}, \dots, \z_{[L]})\) is defined as  
\begin{equation}
    E_{\theta}(\x, \z_{[1]}, \dots, \z_{[L]})=- \mathbf{a}^\top \mathrm{x} -\sum_{i=1}^L \mathbf{b}_i^{\top} \z_{[i]} - \mathrm{x}^\top \mathbf{W}^{(1)}\z_{[1]} - \sum_{i=1}^{L-1}\mathrm{z}_{[i]}^\top \mathbf{W}^{(i+1)}\z_{[i+1]}.\tag{{\tt DBM}}\label{eq:dbm}
\end{equation}
An RBM can be seen as a DBM with a single hidden layer.

\subsubsection{Contrastive Divergence} 
The Contrastive Divergence (CD) algorithm provides a computationally tractable approximation to maximum-likelihood learning in Boltzmann Machines by replacing the intractable model expectation with a short-run Markov chain \citep{hinton2002,carreira-perpinan2005}. The resulting estimator is biased but substantially more efficient, making it suitable for large-scale models where exact sampling would be prohibitively expensive.

The gradient of the log-likelihood of an \ref{eq:RBM} for the interaction terms $\mathbf W$ is 
\begin{align*}
    \frac{\partial}{\partial \mathbf W}\mathcal L(\data,\theta)=\mathbb E_{q(\X)p_{\theta}(\Z|\X)}(\x \z^{\top})-\mathbb E_{p_{\theta}(\X,\Z)}(\x \z^\top)
\end{align*}
where $q(\X)$ is the data distribution. The second term in this expression is computationally intractable. Since an \ref{eq:RBM} has no intra-layer interactions, all units within a layer are conditionally independent given the state of the opposite layer
\begin{align*}
    p_\theta(\x|\z)=\prod_{i=1}^m p_\theta(\x_i|\z)\quad\text{and}\quad
    p_\theta(\z|\x)=\prod_{i=1}^n p_\theta(\z_i|\x).
\end{align*}
For a data point $\x$, the CD-$k$ algorithm approximates the expectation values in the gradients by the sampling scheme
\begin{align*}
    \z_j(t) \sim p_\theta(\z_j = 1 \mid \x(t)) &= \sigma\left( \sum_i \mathbf W_{ij} \x(t)_i + \mathbf b_j\right),\\
    \x(t+1)_{i} \sim p_\theta(\x_i = 1 \mid \z(t)) &= \sigma\left(\sum_j \mathbf W_{ij} \z(t)_j + \mathbf a_i\right),\\
    \mathbb E_{p_{\theta}(\Z|\X=\x)}(\x \z^\top) &\approx \x\z(0)^\top,\\
    \mathbb E_{p_{\theta}(\X,\Z)}(\x \z^\top)&\approx\x(k)\z(k)^\top,
\end{align*}
with $\x(0)=\x$. The gradients for the bias terms are calculated similarly.

The CD approach can be extended to variants of RBMs. In DBMs, as all units within a layer are conditionally independent given the state of the adjacent layers which leads to a sampling scheme similar to RBMs \citep{salakhutdinov2009,salakhutdinov2012}. In a GRBM, the conditional probabilities for the layers are given by
\begin{align*}
    \z_j(t) \sim p_\theta(\Z_j = 1 \mid \x(t)) &= \sigma\left( - \mathbf b_j-\sum_i \mathbf W_{ij} \frac{\x(t)_i}{s_i} \right),\\
    \x(t+1)_{i} \sim p_\theta(\X_i \mid \z(t)) &= \mathrm{Normal}\left(\mu_i+\sum_j \mathbf W_{ij} \z(t)_j ,s_i\right).
\end{align*}
These contrastive divergence techniques made it possible to iteratively train variants of the Boltzmann Machines at scale, paving the way for their use as early generative models.

%% file: Sections/2_quantum.tex
\section{Quantum Information}\label{sec:qi}
Density operators were introduced by \citet{vonneumann1955} as the mathematical foundation of quantum mechanics. The formalism built from them is therefore described as \emph{quantum}, in contrast to the \emph{classical} theory rooted in probability. While this paper does not engage in quantum mechanics directly, the ideas and results carry over to contexts such as quantum computing. We use the adjective quantum to describe the non-probabilistic nature of density operator. This section introduces the foundations of quantum information that are useful to work with the models in this paper.

\subsection{Notation}
This paper works in finite-dimensional Hilbert spaces unless stated otherwise, for the sake of clarity. All results admit natural extensions to the infinite-dimensional setting. A short introduction to matrix analysis is provided in Appendix \ref{app:analysis}.

The absolute value of a complex number $\alpha$ is denoted by $\left|\alpha\right|$. $\hil$ denotes a Hilbert space over $\mathbb{C}$ with dimension $d_{\hil}$. The conjugate transpose of a column vector $\v$ in $\hil$ is given by $\v\h$ and its rank-one projectors are written $\Lambda(\v) = \v\v\h$. 
Operators on $\mathcal{H}$, linear maps that takes vectors to vectors, form the set $\map(\mathcal{H})$. The set of density operators on $\mathcal{H}$ is $\den(\hil)$. The identity operator is denoted $\id$ with the relevant dimension or Hilbert space in the subscript. The support of $A$ is denoted $\mathrm{supp}(A)$ and the null space of $A$ is denoted $\ker(A)$. For operators $A$ and $B$, the Kronecker product is $A \otimes B$ and the direct sum is $A \oplus B$. The Hilbert-Schmidt inner product is $\langle X, Y \rangle = \operatorname{Tr}(X^\dagger Y)$. The commutator of $X$ and $Y$ is denoted $[X,Y]=XY-YX$.

\subsection{Density Operators}
A quantum state is described by a density operator \citep{nielsen2010}.
\begin{definition}[Density Operator]
    Density operators on a Hilbert space $\hil$ is the set $\den(\hil)$ of Hermitian, positive semi-definite operators with unit trace. 
\end{definition}
Normalized elements $\mathbf u$ of the Hilbert space $\hil$ are \textit{pure states} with density operator $\rho=\Lambda(\u)$ while other states are \textit{mixed}. A density operator can be interpreted as a probability distribution over pure states since the spectral theorem stipulates that a density operator $\rho\in\den(\hil)$ can be expressed as a mixture 
\begin{align*}
    \rho= \sum_{i=1}^{d_\hil}\alpha_i \Lambda(\u_i)\quad\text{with}\quad \sum_{i=1}^{d_\hil}\alpha_i=1,\ \alpha_i\geq 0\ \text{and }\u_i\h\u_i = 1\ \text{for all}\ i
\end{align*}
 where $(\alpha_i,\u_i)$ are the eigenvalues and eigenvectors of $\rho$. A density operator over a Hilbert space with dimension $d_\hil$ generalizes the notion of a probability distribution over a sample space with cardinality $d_\hil$. The standard sample space of such a probability space correspond to the standard basis $\{\e_i\}_{i=1}^{d_\hil}$ of the Hilbert space $\hil$. Hence, classical probability distributions are diagonal density operators with $\u_i = \e_i$. It is often useful to assume additional properties for a density operator. 
\begin{definition}[Faithful Density Operator]
    A density operator $\rho$ on a finite-dimensional Hilbert space $\hil$ is \emph{faithful} if the following equivalent conditions hold:
    \begin{multicols}{2}
    \begin{itemize}
        \item $\rho$ is positive definite,
        \item $\rho$ is invertible,
        \item $\mathrm{supp}(\rho)=\hil$,
        \item $\ker(\rho)=0$.
    \end{itemize}
    \end{multicols}
\end{definition}
This manuscript uses the operator theoretic term \emph{faithful} to reflect the conditions required by our theorems in the infinite-dimensional setting \citep{bratteli1987,bratteli1997}.

Density operators on a composite Hilbert space $\hil_A \otimes \hil_B$ play the role of joint probability distributions. A well-studied subclass is that of classical-quantum states, which have been extensively analyzed in the quantum information literature \citep{holevo1998,wilde2016}.
 \begin{definition}[Classical-Quantum State]
     A density operator $\rho$ in $\den(\hil_A\otimes\hil_B)$ is called a \emph{classical-quantum state} if there is a predetermined orthonormal basis $\{\u_i\}_{i=1}^{d_A}$  such that 
     \begin{align*}
         \rho=\sum_{i=1}^{d_A}\alpha_i \Lambda(\u_i) \otimes \rho_{\scriptscriptstyle B}(i)
     \end{align*}
     where $\rho_{\scriptscriptstyle B}(i)$ are density operators in $\den(\hil_B)$ and $\alpha_i\geq 0$ with $\sum_{i=1}^{d_A}\alpha_i=1$.
 \end{definition}
These models are called classical-quantum as the eigenvectors of $\rho$ over the Hilbert space $\hil_A$ is fixed. The component of $\rho$ corresponding to the subsystem $\hil_A$ is fully determined by the classical probability distribution $\{\alpha_i\}_{i=1}^{d_A}$. We use classical-quantum states to formulate classical-quantum latent variable models, a family of \dolvms that can be trained on standard ML data sets on classical hardware, as detailed in \cref{sec:classic}.

\subsubsection{Observables and Measurement}
  Hermitian operators on $\hil$ are the quantum mechanical analogs of random variables. The expected value of an observable $\O$ in $\map(\hil)$ with respect to a density operator $\rho$ is given by the Hilbert-Schmidt inner product $\braket{\mathcal{O}}_\rho=\tr(\rho\,\O)$. This quantity reduces to the classical expectation value if $\rho$ and $\O$ are diagonal as the quantum correlations encoded in the off-diagonal elements does not make an appearance in the expression for the trace. A sample is drawn from $\mathcal O$ using a \textit{measurement}.
  \begin{definition}[Projective Measurement]\label{def:proj}
      A \emph{projective measurement} is described by a Hermitian observable $\O$ in $\map(\hil)$. If the observable has spectral decomposition $\O=\sum_{i=1}^{d_\hil}\lambda_i \Lambda_i$ where $\Lambda_i$ is the projector onto the eigenspace of $\O$ with eigenvalue $\lambda_i$, the measurement results in outcome $\lambda_i$ with probability $\Pr(\lambda_i)=\tr(\rho \Lambda_i)$.
  \end{definition}
  It is possible to identify the projection operator that resulted in the outcome if the observable $\O$ has non-degenerate eigenvalues. Multiple measurements on $\rho$ is physically equivalent to measuring independent copies of same system and result in independent and identically distributed samples. Drawing a sample from a classical random variable $\X$ is equivalent to the measurement of the observable $\O = \sum_{i=1}^{d_\hil}\x_i\Lambda(\e_i)$ on a diagonal density operator $\rho$ which will result in $\x_i$ with probability $\tr(\rho\Lambda(\e_i))$. 

\subsection{CPTP maps and Conditional Operators}
In this section we detail the properties of completely positive trace preserving maps that act on density operators \citep{wilde2016,watrous2018}. A linear map $\N:\map(\hil_A)\to\map(\hil_B)$ is \textit{positive} if $\N(X)$ is positive semi-definite for all positive semi-definite operators $X$ in $\map(\hil)$. Such a map is \textit{completely positive} if $\id_R\otimes\N$ is a positive map for an arbitrary reference system $R$. A map is \textit{trace preserving} if $\tr(\N(X))=\tr(X)$.
\begin{definition}[CPTP map]
     A CPTP map $\N:\map(\hil_A)\to\map(\hil_B)$ is a linear, completely positive, trace preserving map. 
\end{definition}
The \textit{adjoint} of a CPTP map $\N:\map(\hil_A)\to\map(\hil_B)$ is the unique linear map $\N\h:\map(\hil_B)\to\map(\hil_A)$ satisfying $
    \braket{Y,\N(X)}=\braket{\N\h(Y),X}$
for all $X$ in $\map(\hil_A)$ and $Y$ in $\map(\hil_B)$. The adjoint of a CPTP map is completely positive and unital, that is, $\mathcal{N}\h(\id_B)=\mathcal{N}\h(\id_A)$ . 

Projective measurements described in Definition \ref{def:proj} is a CPTP map that takes a density operator to a probability distribution \citep{watrous2018}. Such maps are called \emph{quantum-to-classical} maps. CPTP maps also allow the notion of marginalization of density operators over joint spaces. A density operator defined on a system $\hil_A\otimes\hil_B$ can be \textit{marginalized} to a subspace by the partial trace operation.
\begin{definition}[Partial Trace]
    Let $\{\mathbf u_i\}_{i=1}^{d_{B}}$ be an orthonormal basis of the Hilbert space $\hil_B$. The partial trace of a joint density operator $\rho\in\den(\hil_A\otimes\hil_B)$ with respect to $\hil_B$ is the map $\tr_B:\den(\hil_A\otimes\hil_B)\to\den(\hil_A)$ given by
    \begin{align*}
        \rho_{A}=\tr_B(\rho)=\sum_{i=1}^{d_{B}} (I_A \otimes \mathbf{u}_i\h)\rho (I_A \otimes \mathbf{u}_i).
    \end{align*}
\end{definition}
It is easy to see that the partial trace operation is a CPTP map. The adjoint of the map $\N(\rho)=\tr_B (\rho)$ is given by $\N\h(\tr_B (\rho))=\tr_B (\rho)\otimes \id_B$. Several concepts from probability theory, such as expectation values and marginalization, have direct analogs in the framework of density operators. However, satisfactory counterparts for conditional probability and Bayes’ theorem are still lacking. We outline two existing proposals for such analogs.

\subsubsection{Conditional Probability}
There is no formal notion of conditional probability in density operators \citep{wilde2016}. The conditional amplitude operator was introduced by \citet{cerf1997,cerf1999} as an operator analog to conditional probability. This operator mirrors the role of conditional probability in the definition of quantum conditional entropy.
\begin{definition}[Conditional Amplitude Operator]\label{defn:cao} The conditional amplitude operator of a density operator $\rho$ in $\den(\hil_A\otimes\hil_B)$ with respect to $\hil_A$ is 
\begin{align*}
\rho_{B|A}=\exp(\log \rho-\log \rho_A\otimes \id_B)\tag{\tt CAO}\label{CAO}
\end{align*}
\end{definition}
The \ref{CAO} reduces to conditional probability if $\rho$ is diagonal. In general, the \CAO is not a perfect notion of a conditional density operator as it generally produces trace deficient operators. An EM-like algorithm to learn $\rho_A$ from $\rho_{B|A}$ was conjectured in \citet{warmuth2010}. However, this algorithm did not attempt to learn from data and was not supported by correctness guarantees.

\subsubsection{Bayes' Theorem}
The Petz recovery map, a density operator analog of Bayes' theorem was proposed in \citep{petz1986,petz1988} in the context of recoverability in information theory.
\begin{definition}[Petz Recovery Map]
    The Petz recovery map $\mathcal R_{\N,\rho}:\map(\hil_B)\to\map(\hil_A)$ with respect to a CPTP map $\N:\map(\hil_A)\to\map(\hil_B)$ and a positive semi-definite operator $\rho$ in $\map(\hil_A)$ is
    \begin{align*}
        \mathcal{R}_{\N,\rho}(\omega)=\rho^{1/2}\N\h\left(\N(\rho)^{-1/2}\omega\,\N(\rho)^{-1/2}\right)\rho^{1/2}.\tag{\tt PRM}\label{PRM} 
    \end{align*}
\end{definition}
The \ref{PRM} is a completely positive, trace non-increasing map that does not always map density operators to density operators \citep{wilde2016}. It can be seen from the proof in \citet{wilde2016} that
\begin{align*}
    \tr(\mathcal{R}_{\mathcal{N},\rho}(\omega))&=\tr(\Lambda(\N(\rho))\omega)\leq \tr(\omega)
\end{align*}
where $\Lambda(\N(\rho))$ is the projector onto $\mathrm{supp}(\mathcal{N}(\rho))$. This allows for a condition on when the Petz recovery map is a CPTP map. 
\begin{corollary}\label{thm:fullrank}
    The Petz recovery map $\mathcal R_{\N,\rho}:\map(\hil_B)\to\map(\hil_A)$ with respect to a CPTP map $\N:\map(\hil_A)\to\map(\hil_B)$ and a positive semi-definite operator $\rho$ in $\map(\hil_A)$ is a CPTP map if $\N(\rho)$ is a faithful density operator. 
\end{corollary}
The Petz recovery map reduces to the Bayes' theorem when $\rho$ and $\N$ are probabilistic \citep{wilde2016}. The Petz recovery map with respect to the partial trace is not a perfect notion of Bayes rule as 
\begin{align*}
    \omega \ne \tr_B \left(\mathcal{R}_{\tr_B,\rho}(\omega)\right).
\end{align*}
for $\omega$ in $\den(\hil_A)$ and $\rho$ in $\den(\hil_A\otimes\hil_B)$.
\subsection{Quantum Relative Entropy and Inequalities}
In this section, we collect results about quantum relative entropy that will be useful in the specification of the EM algorithm for density operators.

A real valued function $f:\R\to\R$ can be naturally extended to act on density operators via the spectral theorem, 
\begin{align*}
    f(\rho)\mathord{=}\sum_{i=1}^{d_\hil}f(\alpha_i)\Lambda(\mathbf u_i)
\end{align*}as discussed in \citet{bhatia1997}. This construction was used by \citet{vonneumann1955} to generalize the classical Shannon entropy to the quantum setting, defining the entropy of a density operator as
\begin{align*}
    \svn(\rho)=-\tr(\rho\log\rho).
\end{align*}
This quantity, now called the von Neumann entropy, is only a function of the eigenvalues of the density operator. \cite{umegaki1962} extended the classical Kullback-Leibler divergence \citep{cover2006} between probability distributions to the quantum domain.
\begin{definition}[Umegaki Relative Entropy]
    Let $\omega$ and $\rho$ be density operators in $\den(\hil)$ with $\ker(\rho)\subseteq \ker(\omega)$. Their relative entropy is given by \[\du(\omega,\rho)=\tr(\omega\log\omega)-\tr(\omega\log\rho).\]
\end{definition}
The Umegaki relative entropy inherits several properties of KL divergence such as non-negativity and joint convexity with $\du(\omega,\rho)=0$ if and only if $\omega = \rho$.  \citet{lindblad1975} showed that the relative entropy is monotonic under the action of a CPTP map.
\begin{theorem}[Monotonicity of Relative Entropy]\label{thm:mre}
    For all density operators $\omega$ and $\rho$ in $\den(\hil)$ such that $\ker(\omega)\subset\ker(\rho)$, \(\du(\omega,\rho)\geq \du(\N(\omega),\N(\rho))\).
\end{theorem}
The Monotonicity of Relative Entropy (\MRE) is the quantum information theoretical analog of the data processing inequality. \citet{petz1986,petz1988} showed that \MRE is saturated if and only if the action of the CPTP map is \emph{reversible}. We provide a restatement of this result from \citet{hayden2004}.

\begin{theorem}[Theorem 3, \citeauthor{hayden2004}, \citeyear{hayden2004}]\label{thm:petz}
    For states $\omega$ and $\rho$ in $\den(\hil_A)$ and a CPTP map $\N:\den(\hil_A)\to\den(\hil_B)$, $\du(\omega,\rho)=\du(\N(\omega),\N(\rho))$ if and only if there exists a CPTP map $\mathcal{R}$ such that $\mathcal R(\N(\omega))=\omega$ and $\mathcal R(\N(\rho))=\rho$. Furthermore, on the support of $\N(\rho)$, $\mathcal{R}$ is explicity given by the Petz recovery map.
\end{theorem}
 \citet{ruskai2002} proved an equivalent condition using the matrix logarithm for the saturation of the monotonicity of relative entropy.
\begin{theorem}[Ruskai's condition]\label{thm:ruskai}
    For states $\omega$ and $\rho$ in $\den(\hil_A)$ and a quantum chanenl $\N:\den(\hil_A)\to\den(\hil_B)$, $\du(\omega,\rho)=\du(\N(\omega),\N(\rho))$ if and only if 
    \begin{align*}
        \log \omega - \log \rho =\N\h(\log \N(\omega)-\log \N(\rho)).
    \end{align*}
\end{theorem}
Ruskai’s condition, which can obtained by differentiating the equality condition in Petz’s characterization of monotonicity \citep{petz2003}, provides a criterion for perfect recoverability under the Petz recovery map.

%% file: Sections/3_rw.tex
\section{A Survey on QBMs}\label{sec:qbm}
In this section, we introduce the Quantum Boltzmann Machine, the predominant density operator latent variable model, and outline the main challenges in training it. 

\subsection{Mathematical Qubits}
A mathematical bit used to define Boltzmann machines and other probabilistic models can take the value $+1$ with probability $p$ and the value $-1$ with probability $1-p$. Quantum Boltzmann Machines are defined using mathematical qubits \citep{nielsen2010}, the density operator analog of a mathematical bit.

A mathematical qubit is a density operator over a Hilbert space $\hil$ of dimension $d_\hil=2$ over $\mathbb C$ with the $+1$ state corresponding to the vector $(1,0)\h$ and the $-1$ state corresponding to $(0,1)\h$. The vectors $(1,0)\h$ and $(0,1)\h$ are known as \textit{computational basis states}. Notably, the Hilbert space also contains an infinite number of states $u = \alpha (1,0)\h + \beta (0,1)\h$ with $\alpha$ and $\beta$ being complex numbers such that $|\alpha|^2 +|\beta|^2=1$. When both $\alpha$ and $\beta$ are non-zero, the state $u$ is said to be in \textit{superposition}, a state in the continuum between the computational basis states.
It is well known that such qubits can be represented as unit vectors on a sphere, known as the Bloch sphere, by the parameterization
\begin{align*}
    u = \cos \frac{\theta}{2} \begin{pmatrix}1 \\ 0\end{pmatrix}+e^{i\varphi}\sin \frac{\theta}{2}\begin{pmatrix}0 \\ 1\end{pmatrix}.
\end{align*}

While a mathematical bit can only be sampled for $+1$ or $-1$, a mathematical qubit can be sampled for $+1$ or $-1$ along any direction in the Bloch sphere with each direction giving different statistics. The observables $\O\in\map(\hil)$ that correspond to sampling along the three coordinate axes are given by the Pauli operators \begin{align*}
    \sigma\iter{x}=\begin{pmatrix} 0&1\\1&0\end{pmatrix},\ \sigma\iter{y}=\begin{pmatrix} 0&-i\\i&0\end{pmatrix}, \text{ and } \sigma\iter{z}=\begin{pmatrix} 1&0\\0&-1\end{pmatrix}
\end{align*}
which together with the identity matrix, spans the set of all Hermitian observables on a qubit. The Pauli operators are traceless with eigenvalues $+1$ and $-1$. When sampled along a coordinate axis $k$, a density operator $\rho$ will return $+1$ or $-1$ with average $\tr(\rho\sigma\iter{k})$ where $k\in\{x,y,z\}$. Sampling from a classical bit is equivalent to sampling along the $z$-axis from a diagonal density operator.

\subsection{Quantum Boltzmann Machines}
The Hamiltonian is a Hermitian operator $\H\in\map(\hil)$ representing the total energy and generalizes the notion of an energy function in classical energy-based models \citep{mahan2013}. A Hamiltonian-based model can be defined using the Gibbs state density matrix of a Hamiltonian given by
\begin{align*}
    \rho(\theta)=\frac{\exp(\H(\theta))}{Z(\theta)},\ \ \text{ with }\H(\theta)\mathord{=}\sum_{r} \theta_{r}\H_{r}
    \text{ and }Z(\theta)\mathord{=}\tr \exp({\H(\theta))}
\end{align*}
where $\H_{r}\in\map{(\hil)}$ are fixed Hermitian operators and $\theta_r\in\R$ are model parameters. If $\hil=\hil_\vis\otimes\hil_\hid$ with only $\hil_\vis$ observed, the model is a latent variable model with density operator over the visible units given by $\rho_\vis(\theta) = \tr_\hid \rho(\theta)$. The Quantum Boltzmann Machine (\qbm) introduced by \citet{amin2018} is a Hamiltonian-based model inspired by the transverse field Ising model \citep{mahan2013}. The Hamiltonian of \qbm with $m$ visible units and $n$ hidden units is given by 
\begin{equation*}\tag{\qbm}\label{QBM}
    \H(\theta)=\left\{\begin{aligned}
        &-\sum_{i=1}^{m}\mathbf a_i\sz_{i}-\sum_{i=1}^{n}\mathbf b_i\sz_{m+i}\textcolor{BlueViolet}{-\sum_{i=1}^{m+n}\Gamma_i\sx_i}\\
    &-\sum_{i=1}^m\sum_{j=1}^n\mathbf W_{ij}\sz_i\sz_j-\sum_{i=1}^m\sum_{j=1}^m\mathbf W^{(\vis)}_{ij}\sz_i\sz_j-\sum_{i=1}^n\sum_{j=1}^n\mathbf W^{(\hid)}_{ij}\sz_{m+i}\sz_{m+j}
    \end{aligned}\right.
\end{equation*}
where $\sz_i$ and $\sx_i$ are $2^{m+n}\times 2^{m+n}$ spin operators that act on $\hil_\vis\otimes\hil_\hid$ defined by
\begin{align}\label{eq:pauli}
    \sigma\iter{k}_i=\id_{2^{i-1}}\otimes \sigma\iter{k}\otimes\id_{2^{m+n-i}}\quad \text{for }k\text{ in }\{x,z\}.
\end{align}
The Hamiltonian of the Heisenberg model is also used to define a \qbm \citep{kappen2020}. Setting $\Gamma_i=0$ for all $i$ in this Hamiltonian gives a diagonal Hamiltonian with each entry corresponding to the energy of a state in a classical Boltzmann Machine. These density operator models are plagued by their $2^{m+n}\times 2^{m+n}$ dimensionality. 

\subsection{Learning QBMs}\label{learning}

In this section, we introduce the two learning objectives for Quantum Boltzmann Machines (\qbm) proposed in the literature and survey work done in learning them.

The first objective, termed the \emph{projective log-likelihood} (\ref{PL}), is based on projective measurements in the visible Hilbert space and was initially introduced for QBM training by \citet{amin2018}. Since then, it has been adopted in numerous subsequent works \citep{anschuetz2019, zoufal2021, demidik2025} and is closely related to the likelihood objective employed in quantum state tomography \citep{hradil1997, rehacek2007}. In this framework, the model maximizes the classical log-likelihood of a data set $\data = \{\v\iter{1},\dots,\v\iter{N}\}$ given by
\begin{align*}
    \Lp(\data,\theta) = \frac{1}{N}\sum_{i=1}^N \log\tr\left((\Lambda(\v\iter{i})\otimes \id_\hid)\rho(\theta)\right)=\frac{1}{N}\sum_{i=1}^N \log \tr\left(\Lambda(\v\iter{i})\tr_\hid(\rho(\theta))\right),\tag{{\tt PL}}\label{PL}
\end{align*}
where $\tr\left(\Lambda(\v\iter{i})\tr_\hid(\rho(\theta))\right)$ is the probability of obtaining sample $\v\iter{i}$ under a projective measurement. The presence of projection operators in the learning objective complicates gradient estimation. To circumvent this, \citet{amin2018} introduced a bound-based optimization scheme derived from the Golden-Thompson inequality. However, this approach cannot update the transverse field parameters that control the model’s quantum effects.

The second objective, termed the \emph{quantum log-likelihood} (\ref{QL}), is based on the Umegaki relative entropy and was introduced in \citet{kieferova2017}. A data set with inbuilt quantum correlations or a data set with artificially induced quantum correlations \citep{kappen2020} is used to build an empirical data density operator $\etav$ in $\den(\hil_\vis)$. The maximization of the quantum log-likelihood 
\begin{align*}
    \Lu(\etav,\theta) = \tr\left(\etav\log \tr_\hid \rho(\theta)\right) = \tr\left(\etav\log\rhov(\theta)\right)\tag{{\tt QL}}\label{QL}
\end{align*}
is equivalent to the minimization of the Umegaki relative entropy between $\etav$ and $\rhov(\theta)$. The appearance of the logarithm of a partial trace in this expression prevents the derivation of a closed-form gradient \citep{wiebe2019}. This challenge closely parallels the difficulty encountered in classical latent variable models, where the objective involves a logarithm of a summation. Although both approaches reduce to maximizing the classical log-likelihood when the density operators are diagonal, no formal relationship between the two objective functions has been established to date.

Recent work on QBMs has primarily focused on formulations without hidden units \citep{kappen2020, patel2024, patel2024a, huijgen2024}, owing to the operator-theoretic complications introduced by the partial trace in \ref{QL} during gradient computation. Studies that incorporate hidden variables have relied on heuristic gradient evaluation techniques \citep{amin2018, wiebe2019, anschuetz2019, zoufal2021}, typically involving approximate preparation of Gibbs state density operators \citep{bilgin2010gibbs, chen2023gibbs, ding2025gibbs} without any theoretical guarantees. A related yet orthogonal direction concerns the realization of QBMs on quantum devices, which is a subsequent stage following the algorithmic development of the underlying algorithms \citep{benedetti2017, nelson2022, shibukawa2024}. These methods typically fail to scale beyond toy data sets of 12-bit binary strings, largely because of limited access to quantum hardware, the exponential cost of simulating quantum systems, and the memory demands of representing and optimizing QBMs on classical machines. Even when memory is not the bottleneck, computing the gradient remains a major challenge for models with hidden units. Therefore, the empirical evaluation of these models has been limited to toy data sets, and their performance on realistic data remains an open question.

%% file: Sections/4_dolvm.tex
\section{Density Operator Latent Variable Models}\label{sec:dolvm}
In this section, we undertake a formal study of LVMs arising out of density operators. We prove a fundamental inequality between the differing notions of log-likelihood used to train \qbms and showcase a novel quantum evidence lower bound.

We assume a data set is obtained from $N$ arbitrary projective measurements of an underlying density operator $\psi_\vis$ defined on the Hilbert space $\hil_\vis$. As discussed in \cref{sec:qi}, each projective measurement yields a pure state corresponding to a projection operator, so each sample in the data set can be represented by a unit vector in $\hil_\vis$. The \emph{empirical data density operator} $\etav \in \den(\hil_\vis)$ for a data set $\data = \{\v\iter{1}, \dots, \v\iter{N}\}$ is defined as
\begin{align}\label{eqn:eddo}
\etav = \frac{1}{N}\sum_{i=1}^N \Lambda(\v\iter{i}).
\end{align}
We now formally define a density operator-based LVM.
\begin{definition}
    A Density Operator Latent Variable Model (\dolvm) specifies the density operator $\rhov\in\den(\hil_\vis)$ on observables in $\hil_\vis$ through a joint density operator $\rho\in\den(\hil_\vis\otimes\hil_\hid)$ as $\rhov=\tr_\hid \left(\rho (\theta)\right)$ where the space $\hil_\hid$ is not observed.
\end{definition}

The objective of learning a \dolvm is to obtain parameters $\theta$ such that $\rhov(\theta)$ is as close as possible to a ground truth density operator $\psi_\vis$ from which samples are drawn. This is can be achieved by training the projective log-likelihood (\ref{PL}) or the quantum log-likelihood (\ref{QL}) depending on the data provided. As noted in \cref{learning}, the use of either log-likelihood with hidden units comes with its own distinct set of challenges. The use of \ref{QL}, however, has quantum information theoretic interpretations related to the Umegaki relative entropy and potentially paves the way for employing operator theoretic inequalities to aid the learning process. We now show that for a data set $\data$ arising out of projective measurements, the quantum log-likelihood provides a lower bound to the projective log-likelihood.

\begin{lemma}\label{thm:learning}
    For a data set $\data=\{\v\iter{1},\dots,\v\iter{N}\}$ arising out of projective measurements and a \dolvm $\rho(\theta)$, let the empirical data density operator $\etav$ be as defined in \cref{eqn:eddo}. Then,  $\Lp(\data,\theta)\geq \Lu(\etav,\theta)$.
\end{lemma}
The proof uses stated \cref{thm:dragomir} stated in Appendix \ref{app:analysis}.
\begin{proof}
    The logarithm is a continuous concave function in the positive real and each $\Lambda(\v\iter{i})$ is a positive semi-definite, unit trace operator in $\map(\hil_\vis)$. Then by \cref{thm:dragomir}, we have the inequality
    \begin{align*}
        \log \tr\left(\Lambda(\v\iter{i})\rhov(\theta)\right)\geq \tr\left(\Lambda(\v\iter{i})\log\rhov(\theta)\right)
    \end{align*}
    for each data point $\v\iter{i}\in \data$. Averaging the inequality across all data points results in
    \begin{align*}
        \frac{1}{N}\sum_{i=1}^N\log \tr\left(\Lambda(\v\iter{i})\rhov(\theta)\right)\geq \frac{1}{N}\sum_{i=1}^N\tr\left(\Lambda(\v\iter{i})\log\rhov(\theta)\right),
    \end{align*}
    which by the linearity of trace can be written as
    \begin{align*}
        \frac{1}{N}\sum_{i=1}^N\log \tr\left(\Lambda(\v\iter{i})\rhov(\theta)\right)&\geq \tr\left(\left(\frac{1}{N}\sum_{i=1}^N\Lambda(\v\iter{i})\right)\log\rhov(\theta)\right),\\
        \Lp(\data,\theta)&\geq \Lu(\etav,\theta)
    \end{align*}
    concluding our proof.
\end{proof}
The value of the quantum log-likelihood provides a lower bound on the the projected log-likelihood. However, training on \ref{QL} is also intractable due to the partial trace. 

The intractability of the gradient of the log-likelihood in probabilistic LVMs is addressed by the EM algorithm. As probabilistic LVMs are a special case of density operator latent variable models (\dolvm), the training challenges they face persist in \dolvms. \dolvms also introduce additional operator-theoretic complications due to the absence of conditional density operators and the limited applicability of existing operator versions of Jensen’s inequality. Consequently, the standard derivation of the EM algorithm does not extend naturally to the setting of density operator models. In this light, we investigate possible EM-like algorithms to train \ref{QL} using quantum information. The first step in the derivation of an EM algorithm for \dolvms is the derivation of an evidence lower bound for data operators $\etav$ that can be used in an EM framework. Since the classical ELBO inequality can be derived using the data processing inequality for marginalization, we derive a quantum ELBO using \cref{thm:mre}.

\begin{lemma}[Quantum ELBO]\label{thm:qelbo}
     Let $\mathcal{J}(\etav)=\{\eta\mid \eta\in\den(\hil_\vis\otimes\hil_\hid)\ \&\ \tr_\hid \eta=\etav\}$ be the set of feasible extensions for data $\etav\in\den(\hil_\vis)$. Then for a \dolvm $\rho(\theta)$ and $\eta\in\mathcal{J}(\etav)$, 
    \begin{align}
        \Lu(\etav,\theta)\geq \QELBO(\eta,\theta) =\tr(\eta\log\rho(\theta))+S(\eta)-S(\etav).\tag{{\tt QELBO}}\label{QELBO}
    \end{align}
\end{lemma}

\begin{proof}
    Since $\tr_\hid \eta = \etav$, by \cref{thm:mre} we obtain
    \begin{align*}
        \du(\eta,\rho(\theta))\geq \du(\etav,\rhov(\theta)).
    \end{align*}
    Expanding the expression for Umegaki relative entropy and rearranging the terms results in
    \begin{align*}
        \tr(\eta\log\eta)-\tr(\eta\log\rho(\theta))&\geq \tr(\etav\log\etav)-\tr(\etav\log\rhov(\theta)), \text{ and }\\
        \tr(\etav\log\rhov(\theta))&\geq \tr(\eta\log\rho(\theta))-\tr(\eta\log\eta)+\tr(\etav\log\etav),\\
        \Lu(\etav,\theta)&\geq \tr(\eta\log\rho(\theta)) + \svn(\eta) - \svn(\etav).
    \end{align*}
    which is the required inequality.
\end{proof}
 The \ref{QELBO} expression does not involve projection operators or the partial trace operation that make the optimization of the projected log-likelihood or the quantum log-likelihood hard.  The \ref{QELBO} serves as the density-operator analog of the \ref{ELBO}, since the monotonicity of relative entropy reduces to the data processing inequality when the density operators are diagonal. The classical EM algorithm is a consequence of the evidence lower bound being a minorant of the log-likelihood. As noted in \cref{sec:lvm}, this can be seen as a saturation of the data processing inequality for marginalization. However, it is well known that the monotonicity of relative entropy is often not saturated for the partial trace operation \citep{lesniewski1999,berta2015,wilde2015,carlen2020,cree2022}. Hence, a similar derivation is not directly possible. In the following section, we investigate an expectation-maximization framework for the quantum ELBO.

%% file: Sections/5_doem.tex
\section{The DO-EM Framework}\label{sec:doem}
In this section, we derive an expectation-maximization framework for density operators using \ref{QELBO}. Inspired by the information geometric interpretation of the EM algorithm detailed in \cref{sec:lvm}, we study a quantum information projection problem to saturate \ref{QELBO}.
\subsection{A Quantum Information Projection Problem}

In this section we study the $I$-projection \citep{cover2006} problem for density operators
and show conditions when \ref{PRM} can solve this problem.  
\begin{definition}[Quantum Information Projection]\label{def:qip}
    The Quantum Information Projection of a density operator $\rho$ in $\den(\hil_A\otimes \hil_B)$ onto a density operator $\omega$ in $\den(\hil_A)$ with respect to the partial trace $\tr_B:\den(\hil_A\otimes\hil_B)\to\den(\hil_A)$ is the density operator $\xi^*$ in $\den(\hil_A\otimes\hil_B)$ such that
    \begin{align*}
    \xi^*=\underset{\tr_B(\xi)=\omega}{\argmin}\,\du(\xi,\rho).\label{QIP}\tag{{\tt QIP}}
    \end{align*}
\end{definition}
To the best of our knowledge, this problem has not been studied in literature. We know from \cref{thm:mre} that the theoretical minimum attained by the objective function in \ref{QIP} is $\du(\omega,\tr_B(\rho))$ though it may not always be attained. According to \cref{thm:petz}, the solution to \ref{QIP} should be the Petz recovery map when this lower bound is achieved. Inspired by this connection, we explore sufficiency conditions for when the Petz recovery map is the solution to the \ref{QIP} problem.

As noted in \cref{sec:qi}, the Petz recovery map is not a trace preserving map. The Petz recovery map cannot be a solution to \ref{QIP} if it is not unit trace as it will not partial trace to $\omega$. Hence, the faithfulness of $\tr_\hid(\rho)$ is a necessary condition for the Petz recovery map to sovle \ref{QIP}. We now outline two commutative relations sufficient to ensure that the Petz recovery map solves the \ref{QIP} problem.
\begin{definition}[Sufficient Conditions]
\label{defn:conds} Two density operators $\rho$ in $\den(\hil_A\otimes\hil_B)$ and $\omega$ in $\den(\hil_A)$ satisfy the \emph{sufficient conditions} if: 
\begin{enumerate}
    \item $\tr_\hid(\rho)$ is faithful
    \item $[\rho,\tr_B(\rho)\otimes\id_B]=0$, and
    \item $[\omega,\tr_B(\rho)]=0$.
\end{enumerate}
\end{definition}
The faithfulness condition guarantees that the Petz recovery map provides a valid density operator. The second condition ensures that $\rho$ is block diagonal when $\tr_B(\rho)$ is rotated to be diagonal. The final condition ensures that the $\omega$ and $\tr_B(\rho)$ share the same pure states in their spectral decomposition and only differ in probability of each state, making information projection easier by mapping blocks over $\hil_B$ to the appropriate pure state in $\hil_A$. We defer formal statements on these interpretations to \cref{sec:classic}. When density operators satisfy these sufficiency conditions, we are able to solve the \ref{QIP} problem with the Petz recovery map. We formally state this result in Theorem \ref{thm:qip} below.

\begin{theorem}\label{thm:qip}
    Recall the \ref{QIP} problem in Definition \ref{def:qip}. Suppose $\rho$ and $\omega$ are two density operators in $\den(\hil_A\otimes\hil_B)$ and $\den(\hil_A)$ respectively such that the \conds are satisfied, the solution to the \ref{QIP} problem is the Petz recovery map
    \begin{align*}
        \xi^*=\mathcal{R}_{\tr_B,\rho}(\omega).
    \end{align*}
\end{theorem}
The proof of this theorem uses Lemma \ref{thm:ksum} stated in Appendix \ref{app:expmlogm}.
\begin{proof}
    Since $\tr_\hid(\rho)$ is faithful, the Petz recovery map gives a valid density operator by Lemma \ref{thm:fullrank}. Note that if two matrices commute, their squareroots also commute because they share the same eigenspace. The Petz recovery map with respect to $\rho$  and $\N$ is equal to 
    \begin{align*}
        \mathcal{R}_{\N,\rho}(\omega) = \log \rho^{1/2}\N\h(\N(\rho)^{-1/2}\omega\N(\rho)^{-1/2})\rho^{1/2}.
    \end{align*}
    Since $[\omega,\tr_B(\rho)]=0$, $[\tr_B(\rho)^{-1}\omega,\tr_B(\rho)]=0$ and $[(\tr_B(\rho)^{-1}\omega)\otimes\id_B,\rho]=0$. Therefore, the Petz recovery map is equal to
    \begin{align*}
        \mathcal{R}_{\tr_\hid,\rho}(\omega) &= \rho(\tr_B(\rho)^{-1}\omega\otimes\id_B),\\
        \log \mathcal{R}_{\tr_B,\rho}(\omega) &= \log \rho + \log (\tr_B(\rho)^{-1}\omega\otimes\id_B),\text{ and }&\text{(commutativity)}\\
        \log \mathcal{R}_{\tr_B,\rho}(\omega) &= \log \rho + \log (\tr_B(\rho)^{-1}\omega)\otimes\id_B.&\text{(Lemma \ref{thm:ksum})}
    \end{align*}
    We now show that the Petz recovery map and $\rho$ satisfy the Ruskai condition.
    \begin{align*}
        \log \rho-\log \mathcal{R}_{\tr_B,\rho}(\omega) &= \log\rho - \log \rho - \log (\tr_B(\rho)^{-1}\omega)\otimes\id_B\\
        &=-\log(\tr_B(\rho)^{-1})\otimes\id_B - \log(\omega)\otimes\id_B&\text{(commutativity)}\\
        &=\log(\tr_B(\rho))\otimes\id_B - \log(\omega)\otimes\id_B.&(\log(A^{-1})=-\log A)
    \end{align*}
    Hence, {\tt PRM} is a density operator that satisfies Ruskai's condition in \cref{thm:ruskai}. Therefore, $\du(\mathcal{R}_{\tr_B,\rho}(\omega),\rho)=\du(\omega,\tr_B\rho)$. \cref{thm:petz} implies that $\tr_B \mathcal{R}_{\tr_B,\rho}(\omega)=\omega$. We have shown that the Petz recovery map is both a feasible density operator that partial traces to $\omega$ and solves the \ref{QIP} problem.
\end{proof}
We note that this technique is not generally applicable to arbitrary CPTP maps as Lemma \ref{thm:ksum} does not apply to arbitrary completely positive unital maps. We now develop our EM framework for density operators.

\subsection{DO-EM through the lens of Minorant-Maximization}

In this section, we present the \textbf{D}ensity \textbf{O}perator \textbf{E}xpectation \textbf{M}aximization (\doem) algorithm from a Minorant-Maximization perspective and discuss its advantages over direct maximization of the log-likelihood. We prove that the \doem algorithm can achieve log-likelihood ascent at every iteration when the \conds are satisfied.

For fixed parameters $\theta^{\mathrm{(old)}}$, the \ref{QELBO} is maximized when $\eta$ is the \ref{QIP} of $\rho(\theta)$ onto the set of feasible extensions. This allows us to define a potential minorant $\mathcal Q$ for the log-likelihood. 
\begin{equation*}\label{eq:minorant}
    \begin{aligned}
    \eta{(\theta^{\mathrm{(old)}})}&=\underset{\tr_\hid\eta=\etav}{\argmin}\,\du(\eta,\rho(\theta^{\mathrm{(old)}})) \\
\mathcal{Q}(\theta;\theta^{\mathrm{(old)}})&=\mathrm{QELBO}(\eta{(\theta^{\mathrm{(old)}})},\rho(\theta))
\end{aligned}
\end{equation*}
We use $\mathcal{Q}$ to define the \doem algorithm in \cref{alg:doem}.
\begin{algorithm}[H]
   \caption{Density Operator Expectation Maximization}
   \label{alg:doem}
\begin{algorithmic}[1]
   \STATE {\bfseries Input:} Data density operator $\etav$ and model parameters $\theta^{({0})}$
   \WHILE{not converged}
   \STATE \textbf{E Step:} $\eta^{(t)}= \underset{\eta: \tr_\hid\eta=\etav}{\argmin}\ \du(\eta,\rho(\theta^{(t)}))$
   \STATE \textbf{M Step:} $\theta^{(t+1)}= \underset{\theta}{\argmax}\ \tr(\eta^{(t)}\log\rho(\theta))$
   \ENDWHILE
\end{algorithmic}
\end{algorithm}
As the saturation of monotonicity of relative entropy with respect to the partial trace is not guaranteed, it is not possible to guarantee log-likelihood ascent without additional assumptions. Models and QIPs that obey Ruskai's condition provably achieve log-likelihood ascent under the \doem  procedure. 
\begin{theorem}[$\mathcal Q$ is a minorant]\label{thm:ascent}
    Let $\etav$ be a data density operator and  $\rho(\theta)$ be a \dolvm trained by the \doem  algorithm. If $\rho(\theta^{(t)})$ and its \ref{QIP}, $\eta^{(t)}$, obey the Ruskai condition, then $\mathcal Q$ is a minorant of the log-likelihood. That is,
    \begin{align*}
        \L (\theta^{(t)})\leq \L(\theta^{(t+1)})\text{ where }\theta^{(t+1)}=\argmax_\theta \mathcal Q(\theta;\theta^{(t)}).
    \end{align*}
\end{theorem}
\begin{proof}
    From Ruskai's condition and \cref{thm:qelbo}, we obtain
\begin{align}\label{eq:step1}
    \L(\theta^{(t)})=\mathcal{Q}(\theta^{(t)},\theta^{(t)}).
\end{align}
By the definition of $\theta^{(t+1)}$, 
\begin{align}\label{eq:step2}
    \mathcal{Q}(\theta^{(t)},\theta^{(t)})\leq \mathcal{Q}(\theta^{(t)},\theta^{(t+1)}).
\end{align}
    Combining these inequalities with the QELBO, we obtain
    \begin{equation*}
        \begin{aligned}          
        \mathcal{Q}(\theta^{(t)},\theta^{(t+1)})&\leq \L(\theta^{(t+1)});&\\
\mathcal{Q}(\theta^{(t)},\theta^{(t)})&\leq \L(\theta^{(t+1)})&\text{ by \eqref{eq:step2}};\\
\L(\theta^{(t)})&\leq \L(\theta^{(t+1)})&\text{ by \eqref{eq:step1}}.
        \end{aligned}
    \end{equation*}
\end{proof}
A practical approach to ensure that the model and \ref{QIP} satisfy Ruskai's condition is to ensure that the model and the data operators satisfy the \conds.
\begin{corollary}\label{thm:example}
    For a target density operator $\etav$ and model $\rho(\theta)$ satisfying \conds, the E step is the Petz recovery map $\mathcal{R}_{\rho,\tr_\hid}(\etav)$. Moreover, such a model trained using the \doem  algorithm achieves provable likelihood ascent at every iteration.
\end{corollary}
\begin{proof}
    We proved in \cref{thm:qip} that the Petz recovery map satisfies the Ruskai condition under the \conds. Hence, the statement is an immediate corollary to \cref{thm:ascent}.
\end{proof}
Thus, in the \doem algorithm, the Petz recovery map takes on the role that conditional probability plays in the classical EM algorithm.

\subsection{Commentary on the \doem Algorithm}
The \doem algorithm can be considered a density operator analog of the classical EM algorithm. 
We recover the classical EM algorithm from \doem for discrete models if $\etav$ and $\rho(\theta)$ are diagonal. 

When the \conds are satisfied, the \textbf{E Step} in \doem finds a feasible extension $\eta$ whose Conditional Amplitude Operator ({\tt CAO}) is equal to that of the model $\rho(\theta)$. 
\begin{corollary}
    If two density operators $\eta$ and $\rho$ in $\den(\hil_\vis\otimes\hil_\hid)$ satisfy the Ruskai condition for the partial trace $\tr_B:\den(\hil_\vis\otimes\hil_\hid)\to\den(\hil_\vis)$, then $\eta$ and $\rho$ have the same conditional amplitude operator $\eta_{\hid|\vis}=\rho_{\hid|\vis}$ with respect to the Hilbert space $\hil_\vis$.
\end{corollary}
\begin{proof}
    By the Ruskai condition,
    \begin{align*}
        \log \eta - \log \rho = \N\h(\log\N(\eta)-\log\N(\rho)).
    \end{align*}
    Specializing for $\N(\rho)=\tr_\hid(\rho)$ and $\N\h(\omega)=\omega\otimes\id_\hid$,
    \begin{align*}
        \log \eta - \log \rho &= (\log\tr_\hid(\eta)-\log\tr_\hid(\rho))\otimes\id_\hid,&\\
        \log \eta - \log\tr_\hid(\eta)\otimes\id_\hid &= \log \rho - \log\tr_\hid(\rho)\otimes\id_\hid.&\text{(distributivity)}
    \end{align*}
    Exponentiating both sides, we obtain 
        \begin{align*}
        \exp\left(\log \eta - \log\tr_\hid(\eta)\otimes\id_\hid\right) &= \exp\left(\log \rho - \log\tr_\hid(\rho)\otimes\id_\hid\right),\\
        \eta_{\scriptscriptstyle{\hid|\vis}} =\rho_{\scriptscriptstyle{\hid|\vis}}.
    \end{align*}
    the conditional amplitudes operators as defined in Definition \ref{defn:cao}.
\end{proof}
The Petz recovery map under the \conds is the conditional amplitude operator reweighted by $\etav$ to give a valid density operator. This reduces to classical E step if the operators involved are diagonal as the \ref{CAO} reduces to the conditional probability and {\tt PRM} reduces to Bayes rule. 

A key feature of the classical EM algorithm is the log-of-sums arising out of marginalization in the log-likelihood turns it into a sum-of-logs in the M-step. As discussed in \cref{sec:qbm}, log-likelihoods that involve partial traces or projection operators are often computationally intractable. In contrast, the \textbf{M-step} of the \doem algorithm maximizes an objective that avoids these obstacles entirely. For a Hamiltonian-based model
\begin{equation*}
    \rho(\theta)=\exp(\mathcal{H}(\theta))/Z(\theta), \qquad
\mathcal{H}(\theta)=\sum_r \theta_r \mathcal{H}_r,
\end{equation*}
and an E-step output $\eta^{(t)}$, the gradient of $\mathcal{Q}(\theta;\theta^{\textrm{old}})$ in the M-step with respect to~$\theta_r$ is
\begin{equation}\label{eq:gradient}
  \frac{\partial}{\partial\theta_r}\mathcal{Q}(\theta;\theta^{\textrm{old}})
   = \langle \mathcal{H}_r\rangle_{\eta^{(t)}} 
     - \langle \mathcal{H}_r\rangle_{\rho(\theta)}.  
\end{equation}
Because the M-step maximizes $\tr(\eta^{(t)} \log \rho(\theta))$, its structure matches the log-likelihood of a fully visible model with no hidden units. This allows direct use of standard training techniques for fully visible density-operator models. A proof of the gradient expression, following the Lie-Trotter-Suzuki argument in \citet{kappen2020}, appears as Corollary \cref{thm:LTS} in Appendix \ref{app:mstep}. 

As expected, the M-step optimization is far simpler than computing gradients for direct maximization of the projective or quantum log-likelihoods. As demonstrated in \cref{sec:exp}, the \doem algorithm runs roughly two orders of magnitude faster than directly maximizing the projective log-likelihood for for density operators of size $2^{10}\times 2^{10}$. In the next section, we specialize \dolvms and \doem for classical data, enabling scaling to standard ML data sets. 

%% file: Sections/6_classical.tex
\section{Classical-Quantum Models}\label{sec:classic}
In this section, we focus on adapting \dolvms to standard image data sets in machine learning. Training a \dolvm assumes that the data set was generated by projective measurements on a ground truth density operator $\psi$. We first define such data sets.
\begin{definition}[Classical Data Set]
    A data set is classical if it results from repeated projective measurements of a density operator $\psi \in \den(\hil_\vis)$ using a fixed measurement operator $\mathcal X = \sum_{i=1}^{d_\vis} \x_i \,\Lambda(\u_i)$.
\end{definition}
Standard projective measurements of $\psi$ with respect to $\mathcal X$ correspond to sampling from a probability distribution \citep{watrous2018}. Specifically, outcome $\x_i$ occurs with probability
\begin{align*}
    \Pr(\X = \x_i) = \mathrm{tr}(\psi\,\Lambda(\u_i)).
\end{align*}
Thus, repeated measurements amount to drawing independent and identically distributed samples from this underlying distribution. Standard machine learning data sets can be viewed as arising from such measurements.

For a classical data set, the eigenspace of the empirical data density operator $\etav$ given by \cref{eqn:eddo} coincides with the basis in which the projective measurements were performed. The condition $[\etav, \tr_\hid \rho] = 0$ in the Sufficient Conditions therefore restricts the model class to those whose visible marginals share this same eigenspace. This restriction reflects a fundamental limitation of the data. Since all samples are obtained in a single fixed measurement basis, the observations reveal only the diagonal components of the underlying density operator in that basis. Any coherence or structure present in other orthonormal directions remains completely unobserved and unlearnable for the model. Enforcing that the model align with the measurement basis thus prevents it from trying to infer features that the data does not provide. We now investigate the second commutativity relation in the \conds, $[\rho,\rhov\otimes\id_\hid]=0$, which governs the interaction between the visible and the latent spaces.

\begin{lemma}\label{thm:conditional}
    Consider a density operator $\rho\in\den(\hil_\vis\otimes\hil_\vis)$ such that $\rhov$ has non-degenerate eigenvalues. Then $[\rho,\rhov\otimes\id_\hid]=0$ if and only if $\rho$ is a classical-quantum state
    \begin{align*}
        \rho=\sum_{i=1}^{d_\vis}\alpha_i \Lambda(\u_i) \otimes \rhoh(i).
    \end{align*}
\end{lemma}
\begin{proof}
    First we prove the forward direction: if $[\rho,\rhov\otimes\id_\hid]=0$ and $\rhov$ has non-degenerate eigenvalues, then $\rho$ is of the form $\rho=\sum_{i=1}^{d_\vis}\alpha_i \Lambda(\u_i) \otimes \rhoh(i)$. Let the spectral decomposition of $\rhov$ be given by $\rhov=\sum_{i=1}^{d_{\vis}}\alpha_i\Lambda(\mathbf u_i)$. Since $[\rho,\rhov\otimes\id_\hid]=0$ and $\rhov$ is non-degenerate, $[\rho,\Lambda(\u_i)\otimes\id_\hid]=0$ for all $i$. This gives 
    \begin{align*}
        \rho(\rhov\otimes\id_\hid)-(\rhov\otimes\id_\hid)\rho=0,\\
        \sum_{i=1}^{d_\vis}\alpha_i \left(\rho(\Lambda(\u_i)\otimes\id_\hid)-(\Lambda(\u_i)\otimes\id_\hid)\rho\right)=0.
    \end{align*}
    Multiplying by $\Lambda(\u_j)\otimes\id_\hid$ on the left and $\Lambda(\u_k)\otimes\id_\hid$ on the right, we obtain
    \begin{align*}
        (\alpha_i-\alpha_j)\left((\Lambda(\u_j)\otimes\id_\hid)\rho(\Lambda(\u_k)\otimes\id_\hid\right)&=0,&\\
        \left((\Lambda(\u_j)\otimes\id_\hid)\rho(\Lambda(\u_k)\otimes\id_\hid\right)&=0,&(\alpha_j\ne \alpha_k)
    \end{align*}
    for all $j\ne k$. Hence,
    \begin{align*}
        \rho &= \sum_{i=1}^{d_\vis}\sum_{j=1}^{d_\vis}(\Lambda(\u_i)\otimes\id_\hid)\rho(\Lambda(\u_j)\otimes\id_\hid),\\
        &=\sum_{i=1}^{d_\vis}(\Lambda(\u_i)\otimes\id_\hid)\rho(\Lambda(\u_i)\otimes\id_\hid).
    \end{align*}
    This is equivalent to
    \begin{align*}
        \rho=\sum_{i=1}^{d_\vis}\alpha_i\Lambda(\u_i)\otimes\rhoh(i).
    \end{align*}
    Since $\rho$ and $\rhov$ are unit trace, $\rhoh(\x_i)$ are unit trace. Similarly, $\rhoh(\x_i)$ is positive semi-definite.

    We now prove the reverse direction: if $\rho$ is of the form $\rho=\sum_{i=1}^{d_\vis}\alpha_i\Lambda(\mathbf{u}_i)\otimes\rhoh(i)$, then $[\rho,\rhov\otimes \id_\hid]=0$. Applying the partial trace to $\rho=\sum_{i=1}^{d_{\vis}} \alpha_i\Lambda(\mathbf{u}_i)\otimes\rhoh(\x_i)$ results in 
    \begin{align*}
    \rhov=\tr_\hid \sum_{i=1}^{d_{\vis}} \alpha_i\Lambda(\mathbf{u}_i)\otimes\rhoh(i)=\sum_{i=1}^{d_{\vis}} \alpha_i\Lambda(\mathbf{u}_i).
    \end{align*}
    Since $[\rhov,\Lambda(\u_i)]=0$ and $[\id_\hid,\rhoh(i)]=0$ for all $i$, we conclude that $[\rho,\rhov\otimes\id_\hid]=0$.
\end{proof}
\dolvms that satisfy the \conds for a classical data set must therefore be classical-quantum states. Furthermore, since the partial trace of the model is required to share the same eigenbasis as the measurement operator, we define the following class of \dolvms.
\begin{definition}[CQ-LVM]
    A \dolvm is a CQ-LVM with respect to a classical data set generated by the measurement $\mathcal X = \sum_{i=1}^{d_\vis} \x_i \,\Lambda(\u_i)$ if it can be expressed as
    \begin{align*}\tag{{\tt CQ-LVM}}\label{CQ}
        \rho(\theta)=\sum_{i=1}^{d_\vis}\Pr(\X\eq\x_i|\theta) \Lambda(\u_i) \otimes \rhoh(i|\theta),
    \end{align*}
    where $\rhoh(i|\theta)$ are in $\den(\hil_\hid)$.
\end{definition}

It is easy to see that all \qlvms satisfy the \conds even if $\rhov$ has degenerate eigenvalues. The special structure of these models allow us to write specialized expressions for quantum log-likelihood and the Petz recovery map.
\begin{lemma}\label{thm:cqelbo}
    Consider a classical data set $\data=\{\x\iter{1},\dots,\x\iter{N}\}$ from projective measurements with $\mathcal X = \sum_{i=1}^{d_\vis} \x_i \,\Lambda(\u_i)$ and a \ref{CQ} $\rho(\theta)=\sum_{i=1}^{d_\vis}P(\X\eq\x_i|\theta)\Lambda(\u_i)\otimes\rhoh(\x_i|\theta)$. The quantum log-likelihood of this model is given by $\Lu(\etav,\theta)=\frac{1}{N}\sum_{i=1}^N \log P(\X=\x\iter{i}|\theta)$. The Petz Recovery Map for $\etav$ with respect to $\rho(\theta\iter{\textrm{old}})$ is 
    \begin{align*}
        \mathcal R_{\rho,\tr_\hid}(\etav) = \frac{1}{N}\sum_{i=1}^N \Lambda(\v\iter{i})\otimes \rhoh(\x\iter{i}|\theta\iter{\textrm{old}}).
    \end{align*}
\end{lemma}
\begin{proof}
    As $\rho(\theta)$ is a \ref{CQ} with respect to the basis $\B$, it satisfies the \conds and $[\etav,\rhov(\theta)]=0$. Since the matrices commute and can be simultaneously diagonalized, $\Lu(\etav,\theta)=\tr(\etav\log\rhov(\theta))$ reduces to a function of the eigenvalues of the matrices,
    \begin{align*}
        \Lu(\etav,\theta)=\sum_{i=1}^N \log P(\X\eq\x\iter{i}|\theta).
    \end{align*}
    Since the \conds are satisfied, the Petz Recovery Map can be computed using commutativity as follows
    \begin{align*}
        \mathcal R_{\rho,\tr_\hid}(\etav)&=\rho^{-1/2}\left(\left(\tr_B(\rho)^{-1/2}\etav\tr_B(\rho)^{-1/2}\right)\otimes\id_B\right)\rho^{-1/2},&\\
        &=\rho(\rhov^{-1}\otimes\id_\hid)(\etav\otimes\id_\hid),&\text{(commutativity)}\\
        &=\left(\sum_{i=1}^{d_\vis}\Lambda(\u_i)\otimes\rhoh(\x_i)\right)(\etav\otimes\id_\hid),&\text{(spectral theorem)}\\
        &=\left(\sum_{i=1}^{d_\vis}\Lambda(\u_i)\otimes\rhoh(\x_i)\right)\left(\sum_{i=1}^N\Lambda(\v\iter{i})\otimes\id_\hid\right),&\text{(spectral theorem)}\\
        &=\frac{1}{N}\sum_{i=1}^N \Lambda(\v\iter{i})\otimes \rhoh(\x\iter{i}).&\text{(distributivity)}
    \end{align*}
\end{proof}
Since all data points are results of the same projective measurement, the log-likelihood and the Petz Recovery Map can be calculated for each data point. Using an analogous result to the chain rule of KL divergence for \qlvms proved as \cref{thm:crure} in Appendix \ref{app:crure}, we decompose the \ref{QELBO} into smaller terms corresponding to each data point.
\begin{theorem}[Decomposition of CQ-ELBO]
     For a classical data set and a corresponding \ref{CQ} $\rho(\theta)$, the minorant function can be decomposed as $\mathcal{Q}(\theta;\theta\iter{\textrm{old}})=\frac{1}{N}\sum_{i=1}^N \mathcal{Q}_i(\theta;\theta\iter{\textrm{old}})$ where
        \begin{align*}
        \mathcal{Q}_i(\theta;\theta\iter{\textrm{old}}) =  \tr\left(\rhoh(\x\iter{i}\mid\theta\iter{\textrm{old}})\log \left(\Pr(\X\eq\x\iter{i}\mid\theta)\rhoh(\x\iter{i}\mid\theta)\right)\right) + \svn\left(\rhoh(\x\iter{i}\mid\theta\iter{\textrm{old}})\right).
    \end{align*}
\end{theorem}
\begin{proof}
        Plugging the Petz Recovery Map into \cref{eq:minorant} allows the following simplification
    \begin{align*}
        \mathcal{Q}(\theta,\theta\iter{\textrm{old}})&=\mathrm{QELBO}(\eta{(\theta^{\mathrm{(old)}})},\rho(\theta)),&\\
        &=-\du(\eta(\theta^{\mathrm{(old)}}),\rho(\theta))-\svn(\etav),&\text{(definition)}\\
        &= -\du(\etav,\rhov(\theta))-\frac{1}{N}\sum_{i=1}^{N}\du(\rhoh(\theta^{\mathrm{(old)}}),\rhoh(\theta))-\svn(\etav),&\text{(\cref{thm:crure})}\\
        &= \tr(\etav\log\rhov(\theta))-\frac{1}{N}\sum_{i=1}^{N}\du(\rhoh(\theta^{\mathrm{(old)}}),\rhoh(\theta)),&\text{(definition)}\\
        &= \frac{1}{N}\sum_{i=1}^N \log \Pr(\x\iter{i}\mid\theta) -\frac{1}{N}\sum_{i=1}^{N}\du(\rhoh(\theta^{\mathrm{(old)}}),\rhoh(\theta))&\text{(commutativity)}\\
        &= \frac{1}{N}\sum_{i=1}^N \mathcal{Q}_i(\theta,\theta\iter{\textrm{old}})
    \end{align*}
    as required.
\end{proof}
It is straightforward to verify that $\mathcal{Q}_i(\theta,\theta\iter{\textrm{old}})$ is a minorant of $\log \Pr(\X=\x\iter{i}|\theta)$. Moreover, if the $\rhoh(i|\theta)$ were diagonal, this expression reduces to the classical \ref{ELBO}. \cref{thm:cqelbo} allows us to specialize \cref{alg:doem} to \qlvms as a mirror of the classical EM algorithm.

\begin{algorithm}[H]
   \caption{\doem for \qlvm}
   \label{alg:cqdoem}
\begin{algorithmic}[1]
   \STATE {\bfseries Input:} $\data=\{\v\iter{1},\dots,\v\iter{N}\}$ and $\theta^{({0})}$
   \WHILE{not converged}
   \FOR{$i=1$ to $N$}
   \STATE $\mathcal{Q}_i(\theta;\theta^{(t)})=\tr\left(\rhoh(\x\iter{i}\mid\theta\iter{t})\log P(\x\iter{i}\mid\theta)\rhoh(\x\iter{i}\mid\theta)\right) + \svn(\rhoh(\x\iter{i}\mid\theta\iter{t}))$
   \ENDFOR
   \STATE $\theta^{(t+1)}=\argmax_\theta \frac{1}{N}\sum_{i=1}^N \mathcal{Q}_i(\theta;\theta^{(t)})$
   \ENDWHILE
\end{algorithmic}
\end{algorithm}
\noindent
\cref{alg:cqdoem} allows us to iterate through the dataset without any dependency on the size of $\etav$ as the measurement basis is fixed. If the model $\rho(\theta)$ can be tractably computed, \doem allows for the tractable learning of \qlvms. The inequality in Lemma \ref{thm:learning} is saturated for \qlvms. Hence, the \doem algorithm provably trains both the quantum and projective log-likelihoods. We now show that \qlvms are a broad class of models that include several Hamiltonian-based models.

\subsection{CQ Hamiltonian-based models and Contrastive Divergence}
In this section, we develop Hamiltonian-based \qlvms and techniques to train them at scale. We assume that the data $\mathcal  D=\{\mathrm x\iter{1},\dots,\mathrm x\iter{N}\}$ is sampled from the set $S=\{+1,-1\}^m$. The corresponding Hilbert space $\hil_\vis$ then has dimension $2^m$ with projective measurements along the standard basis. We begin by deriving a condition on the Hamiltonian that must be satisfied for a Hamiltonian-based model to  be a \qlvm.

\begin{lemma}\label{thm:taylor}
    A Hamiltonian-based model $\rho(\theta)=e^{\H(\theta)}/Z(\theta)$ with $\H(\theta)=\sum_r\theta_r\H_r$ is a \ref{CQ} with respect to the basis $\{\e_i\}_{i=1}^{2^m}$ if and only if $\H = \oplus_i \H_{i}$ where $\H_{i}$ are Hermitian operators in $\map(\hil_\hid)$.
\end{lemma}
\begin{proof}Note that a \ref{CQ} with respect to the basis $\{\e_i\}_{i=1}^{2^m}$ can be expressed as
\begin{align}\label{eq:oplus}
    \rho = \sum_{i=1}^{2^m} \Pr(\X\eq\x_i)\Lambda(\mathbf{e}_i)\otimes\rhoh(\x_i) = \oplus_{i=1}^{2^m}\Pr(\x_i)\rhoh(\x_i).
\end{align}

    We first prove that a Hamiltonian of the form $\H = \oplus_i \H_{i}$ leads to a model of the form \cref{eq:oplus}. The Taylor series of the matrix exponential is given by $e^X = \sum_{k=0}^\infty \frac{1}{k!}X^k$. Since $\H$ is a block diagonal matrix, the powers of $\H$ are also block diagonal. Therefore $\rho$ is also a block diagonal matrix as required. 

    We now prove that a model of the form \cref{eq:oplus} arises from a Hamiltonian of the form $\H = \oplus_i \H_{i}$. The Taylor series of the matrix logarithm is given by $\log X = - \sum_{k=1}^\infty \frac{1}{k}(I-X)^k$. If $(\log Z(\theta))\rho=e^{\H(\theta)}$ is a block diagonal matrix, then its logarithm $\H(\theta)$ is also a block diagonal matrix as required. 
\end{proof}
We now recall QBMs from \cref{sec:qbm}.
This block diagonal structure of the model allows the specialization of \qbms to \qlvms.
\begin{corollary}\label{thm:er}
    A $\qbm$ with $m$ visible and $n$ hidden units given by the Hamiltonian
    \begin{equation*}
    \H(\theta)=\left\{\begin{aligned}
        &-\sum_{i=1}^{m}\mathbf a_i\sz_{i}-\sum_{i=1}^{n}\mathbf b_i\sz_{m+i}-\sum_{i=1}^{m+n}\Gamma_i\sx_i\\
    &-\sum_{i=1}^m\sum_{j=1}^n\mathbf W_{ij}\sz_i\sz_j-\sum_{i=1}^m\sum_{j=1}^m\mathbf W^{(\vis)}_{ij}\sz_i\sz_j-\sum_{i=1}^n\sum_{j=1}^n\mathbf W^{(\hid)}_{ij}\sz_{m+i}\sz_{m+j}
    \end{aligned}\right.
\end{equation*}
    is a \qlvm if and only if quantum terms on the visible units are zero.
\end{corollary}
\begin{proof}
\cref{thm:taylor} implies that the Hamiltonian is \class if and only if it is a direct sum of $2^m$ matrices of size $2^n\times 2^n$. The only off-diagonal entries of the Hamiltonian are from the quantum bias terms $\Gamma_i$. Recall that $\sx_i$ are defined as
\begin{equation*}
    \sigma^k_i=\otimes_{j=1}^{i-1}\id\otimes \sigma^k\otimes_{j=i+1}^{m+n}\id
\end{equation*}
where $\sx=\begin{psmallmatrix}0&1\\1&0\end{psmallmatrix}$ and $\id=\begin{psmallmatrix}1&0\\0&1\end{psmallmatrix}$. It is easy to see that $\sx_i=\oplus_{j=1}^{i}\left(\sigma^k\otimes_{j=i+1}^{m+n}\id\right)$ for $i>m$ which satisfies the condition in \cref{thm:taylor}.

To prove the converse, we observe that the Kronecker product $\sx\otimes M=\begin{psmallmatrix}0&M\\M&0\end{psmallmatrix}$ for any square matrix $M$. When $i\leq m$, this results in matrices larger than $2^n\times 2^n$ with non-zero entries which cannot then be represented as a direct sum of $2^n\times 2^n$ matrices. Hence, $\Gamma_i=0$ if $i\leq m$.
\end{proof}
While, we showcase one set of Hamiltonians that are \qlvms, several other Hamiltonians leading to \qlvms can be constructed. Notably, the class of semi-quantum RBMs studied in \citet{demidik2025} are \qlvms. While the Hamiltonians of \qlvms have additional structure, training such models using real-world data is intractable since the free energy term, $-\log Z(\theta)$, cannot be easily computed even for classical Boltzmann machines. 

In Boltzmann machines, this issue is typically side-stepped by the Contrastive Divergence algorithm discussed in \cref{sec:lvm}. To achieve tractable training of \qbms, we adapt the CD algorithm to the Hamiltonian of a Quantum RBM ({\tt QRBM}) with transverse field only in the hidden units given by
\begin{align}\tag{{\tt QRBM}}\label{QRBM}
    \H(\theta) = -\sum_{i=1}^{m}\mathbf a_i\sz_i-\sum_{i=1}^{n}\mathbf b_i\sz_{m+i} -\sum_{i=1}^{m}\sum_{j=1}^{n}\mathbf W_{ij}\sz_i\sz_{m+j} -\sum_{i=1}^{n}\Gamma_i\sz_{m+i}
\end{align}
where $\mathbf a$ in $\R^m$ and $\mathbf b$ in $\R^n$ are the usual bias terms, $\mathbf W$ in $\R^{m\times n}$ is the weight matrix between the two layers, and $\Gamma$ in $\R^n$ is the transverse bias term. If the non-quantum visible layer is fixed to $\x$, with corresponding Hilbert space vector $\v$, the hidden units of the hidden layer are conditionally independent. The Hamiltonian of the $j^{\mathrm{th}}$ unit of the hidden layer $\mathrm L$ is given by 
\begin{align*}
    \H_{\hid}(j|\x,\theta)=-\mathbf b_j^{\mathrm{eff}}\sz-\Gamma_j\sx
\end{align*}
with $\mathbf b_j^{\mathrm{eff}}=\mathbf b_j+\sum_{i=1}^m \mathbf W_{ij}\x_i$. The Hamiltonian decomposes into $n$ independent $2\times 2$ matrices and allows for the tractable sampling from the quantum layer using the expected values
\begin{equation}\label{eq:gibbs}
\begin{aligned}
        \braket{\sz_j}_{\mathbf{v}}=\frac{\mathbf b_j^{\mathrm{eff}}}{D_j}\tanh{D_j}\ \text{and}\ 
        \braket{\sx_j}_{\mathbf{v}}=\frac{\Gamma_j}{D_j}\tanh{D_j}
\end{aligned}
\end{equation}
where $D_j=\sqrt{(\mathbf b_j^{\mathrm{eff}})^2+\Gamma_j^2}$. The Gibbs step for the non-quantum layers can be performed as per the classical CD algorithm using the quantum sample along the $Z$-basis. This closed-form expression for Gibbs sampling without matrices allows CD to run on a {\tt QRBM} with the same memory footprint as an RBM. This scheme can easily be generalized to the class of models considered in \citep{demidik2025}. We provide the justification for the sampling step stated here in Appendix \ref{app:sampling}.

\subsection{Novel QBM Variants}
In this section, we extend the notion of quantum RBMs to Deep Boltzmann Machines and Gaussian-Bernoulli RBMs introducded in \cref{sec:lvm} and adapt the Gibbs sampling step described above to these new models.

\subsubsection{Quantum Interleaved DBM}
A \textbf{Q}uantum \textbf Interleaved \textbf Deep \textbf Boltzmann \textbf Machine (\qidbm) is a Deep Boltzmann Machine with quantum bias terms on \textbf{non-contiguous hidden layers}. It can be viewed as a set of stacked classical-quantum RBMs. We describe the Hamiltonian of a three-layered $\qidbm$ with $\ell$ visible units and $m$ and $n$ hidden units respectively in the two hidden layers. 
For ease of presentation, the quantum bias terms are present in the middle layer.
\begin{equation*}\tag{{\tt QiDBM}}\label{QiDBM}
    \H(\theta)=\left\{\begin{aligned}[c]&-\sum_{i=1}^m \mathbf a_i\sz_i-\sum_{i=1}^{m}\mathbf b^{(1)}_i\sz_{\ell+i}-\sum_{i=1}^{n}\mathbf b^{(2)}_i\sz_{\ell+m+i}-\sum_{{i=1}}^{m}\Gamma_i\sx_{\ell+i}\\
    &-\sum_{i=1}^\ell\sum_{j=1}^{m} \mathbf W^{(1)}_{ij}\sz_i\sz_{\ell+j}-\sum_{i=1}^{m}\sum_{j=1}^{n} \mathbf W^{(2)}_{ij}\sz_{\ell+i}\sz_{\ell+m+j}
\end{aligned}\right.
\end{equation*}
 We illustrate the case of the middle layer of $\qidbm_{\ell,m,n}$. If the non-quantum visible and hidden layers are fixed to $\x$ and $\z_{[2]}$, with corresponding Hilbert space vectors $\mathbf{v}$ and $\mathbf{h}^{(2)}$, the hidden units of the quantum layer are conditionally independent. The Hamiltonian for the $j^{\mathrm{th}}$ unit of the quantum layer is given by $\H^{{(1)}}(j|\x,\z_{[2]},\theta)=-\mathbf b_j^{\mathrm{eff}}\sz-\Gamma_i\sx$.
We sample from the quantum layer using the expected values
\begin{equation*}
\begin{aligned}
        \braket{\sz_j}_{\x,\z_{[2]}}=\frac{\mathbf b_j^{\mathrm{eff}}}{D_j}\tanh{D_j}\ \text{and}\ 
        \braket{\sx_j}_{\x,\z_{[2]}}=\frac{\Gamma_j}{D_j}\tanh{D_j}
\end{aligned}
\end{equation*}
where $D_j=\sqrt{(\mathbf b_j^{\mathrm{eff}})^2+\Gamma_j^2}$ and $b_j^{\mathrm{eff}}=b_j+\sum_{i=1}^\ell \mathbf W^{(1)}_{ij}\x_i+\sum_{i}^n \mathbf W^{(2)}_{ji}\z_{[2],i}$. The Gibbs step for the non-quantum layers is done as per the classical CD algorithm using the quantum sample from the $Z$-basis.

\subsubsection{Quantum Gaussian-Bernoulli RBM}
A \textbf{Q}uantum \textbf{G}aussian-Bernoulli \textbf{R}estricted \textbf Boltzmann \textbf Machine ({\tt QGRBM}) is a \ref{eq:GRBM} modified to include the transverse field over the Hidden units. Since the visible data is sampled from $\mathbb R$, the visible Hilbert space $\hil_\vis$ should be infinite-dimensional. All separable infinite-dimensional Hilbert spaces over $\mathbb C$ has dimension equal to the cardinality of the real numbers \citep{Conway1990}. We fix \(\hil_\vis\) as the space of square-summable complex sequences and associate each real number \(\x\) with a corresponding basis vector \(\mathbf{x}\) in an orthonormal basis of \(\hil_\vis\). 

The Hamiltonian of the finite-dimensional space $\hil_\hid$ with respect to $\x$ in $\mathbb R^m$ is 
\begin{equation*}\tag{{\tt QGRBM}}\label{QGRBM}
    \H(\mathrm x,\theta)=-\sum_{i=1}^m \frac{(\x_i-\mathbf a_i)^2}{\mathbf s_i^2}\id_\hid -\sum_{j=1}^n \mathbf{b}^{\mathrm{eff}}_j \sz_j - \sum_{j=1}^n \Gamma_j \sx_j,
\end{equation*}
where $\mathbf{b}^{\mathrm{eff}}_j=\mathbf{b}_j+\sum_{i=1}^m \mathbf W_{ij}\frac{\x_i}{s_i}$. The model density operator $\rho$ is then obtained by integrating over the visible space
\begin{align*}
    \rho(\theta)&=\frac{1}{\mathcal{Z}(\theta)}\int_{\mathbf x} \Lambda(\mathbf x)\otimes \exp(\H(\x,\theta))\mathrm d\mathbf x,\\
    \mathcal{Z}(\theta)&=\int_{-\infty}^{+\infty} \tr\exp(\H(\x,\theta))\mathrm{d}\x.
\end{align*}
The probability density function of the visible units is given by $\Pr(\x|\theta)=\frac{\tr\exp(\H(\x,\theta))}{\mathcal{Z}(\theta)}$. If the visible layer is fixed to $\x$, the Hamiltonian for the the $j^{\mathrm{th}}$ unit of the hidden layer $\mathrm L$ is given by $\H_{\hid}(j|\x,\z_{[2]},\theta)=-\mathbf b_j^{\mathrm{eff}}\sz-\Gamma_i\sx$. Much like in the {\tt QRBM} and {\tt QiDBM} models, sampling over these qubits remains efficient, allowing CD to approximate the gradients. The visible units can are drawn from the Gaussian distribution defined in \cref{sec:lvm}, using the hidden-qubit samples obtained in the $Z$-basis.

%% file: Sections/7_exp.tex
\section{Experiments}
\label{sec:exp}

In this work, we introduced Classical-Quantum LVMs---a class of \dolvms---and a general EM framework, \doem, for learning them. In this section, we empirically evaluate our methods with experiments designed to answer the following questions.

\begin{questions}[leftmargin=*,label=({Q}{{\arabic*}})]
    \item\label{Q:PRAC}\textbf{Effectiveness of DO-EM.} Do \qlvms trained with \cref{alg:doem} provide any computational advantage over gradient-based approaches in training \dolvms?\item\label{Q:SCALE}\textbf{Scalability of DO-EM.} Can \qlvms trained with \cref{alg:cqdoem} scale to standard ML data sets such as MNIST?
    \item\label{Q:PERF}\textbf{Performance of DO-EM.} Do \qlvms trained with \cref{alg:cqdoem} provide a measurable improvement over probabilistic LVMs on standard ML data sets?
\end{questions}
To address these questions, we train variants of the \qbm proposed in this work and compare them against benchmarks reproduced using reference implementations from literature. Code for all experiments is available at \href{https://github.com/aditvishnu/DO-EM}{{\tt https://github.com/aditvishnu/DO-EM}}. Details of the computing infrastructure used for all experiments are provided in Appendix \ref{app:hardware}. 

\subsection{DO-EM vs Gradient-based Maximization of Projective Log-Likelihood}

We conduct experiments using \emph{exact log-likelihood computation} to compare \doem with gradient-based maximization of the projective log-likelihood. A \qbm satisfying the \conds and trained using \cref{alg:doem} is compared against a semi-restricted \qbm trained via gradient-based maximization of projective log-likelihood (\ref{PL}), following the approach of \citet{amin2018}.

The \qbms are trained on a mixture of Bernoulli data set introduced in \citet{amin2018}. A Bernoulli distribution is built on a randomly selected mode and the average probability distribution over $M$ such modes determines the mixture of Bernoulli distribution. The distribution can be written as
\begin{align*}
P_v^{\mathrm{data}}=\frac1M\sum_{k=1}^M p^{m-d^k_v}(1-p)^{d^k_v}\label{eq:mbd}
\end{align*}
where $m$ is the dimension of data, $p$ is the probability of success, and $d^k_v$ is the Hamming distance between $v$ and the randomly selected mode. 
We use 100 randomly sampled data sets of 1000 samples each with $N=1000$, $M=8$, $m=8$, and $p=0.9$. A model with $m=8$ visible units and $n=2$ hidden units with random initialization is trained on all 100 datasets with a learning rate of $0.1$, and we plot the resulting average log-likelihood in \cref{fig:prac}. For the semi-restricted QBM, where the Projective and Umegaki log-likelihoods differ, we show both curves in the plot.

\cref{alg:doem} executes a pair of EM steps in only \emph{0.2 seconds}, while direct gradient ascent on the projective log-likelihood requires approximately \emph{40 seconds} per step. Thus, \doem offers a significant speed-up for training \qbms in the exact computation setting. This efficiency can be further enhanced by employing \cref{alg:cqdoem} to train the \qbm. As shown in \cref{fig:prac}, the model trained with \doem \emph{converges in roughly half the number of steps} compared to direct gradient ascent while reaching \emph{comparable log-likelihood values}. These findings position \doem as a strong candidate for \qbm training, offering clear empirical gains in computational resources while maintaining similar performance.

\begin{figure}[t]
    \begin{subfigure}[t]{0.32\linewidth}
        \centering
        \includegraphics[width=\linewidth]{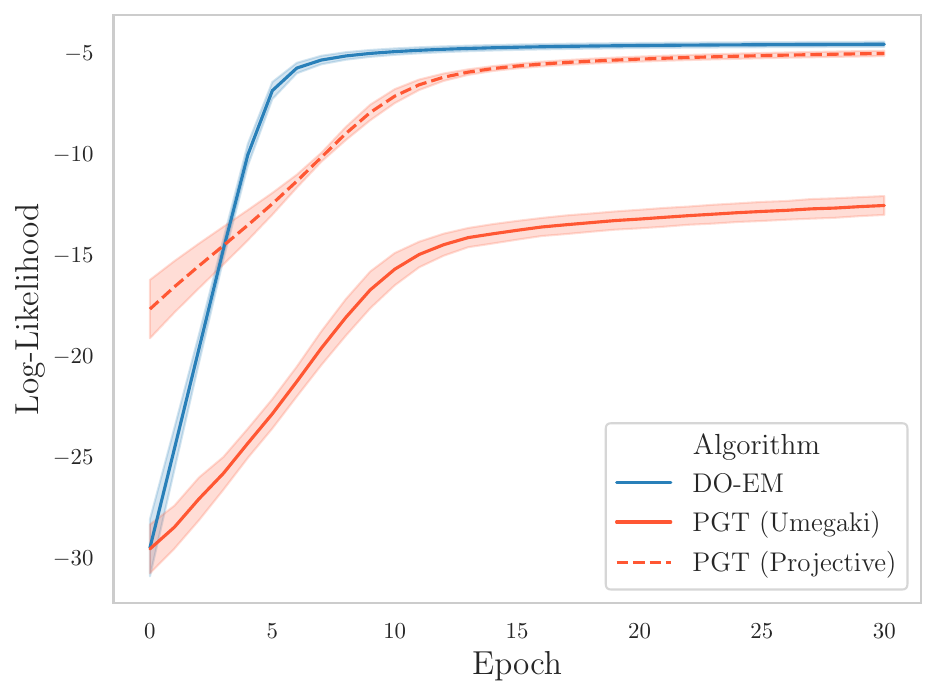}
        \caption{Log-likelihood during training with exact computation showing \doem achieving comparable results faster.}
        \label{fig:prac}
    \end{subfigure}
    \hfill
    \begin{subfigure}[t]{0.32\linewidth}
        \centering
        \includegraphics[width=\linewidth]{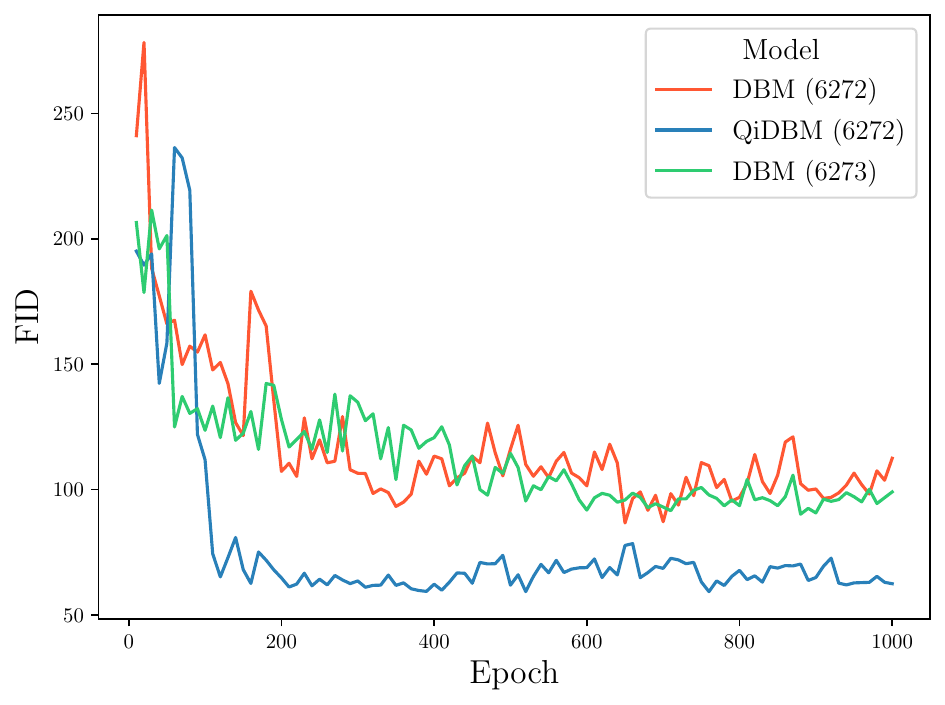}
        \caption{FID on MNIST on a QiDBM (6272 units per hidden layer) and DBMs (6272 and 6273 units per hidden layer).}
        \label{fig:scale}
    \end{subfigure}
    \hfill
    \begin{subfigure}[t]{0.32\linewidth}
        \centering
        \includegraphics[width=\linewidth]{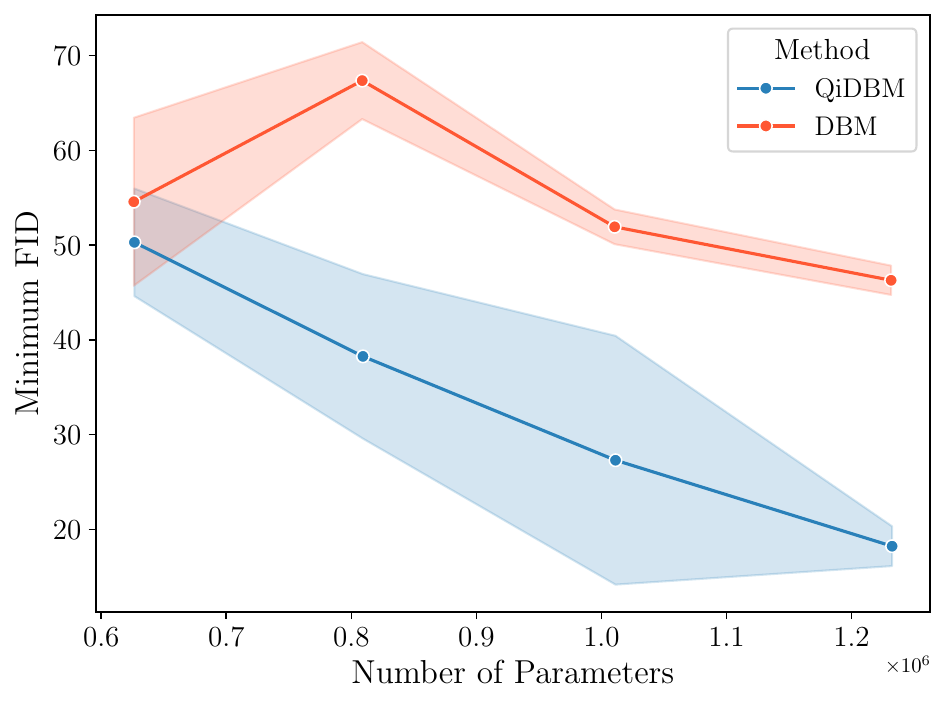}
        \caption{FID scores on Binarized MNIST as a function of model parameters for QiDBM and DBM (averaged over 3 seeds).}
        \label{fig:perf}
    \end{subfigure}
    \caption{(a) \doem vs Proejctive log-likelihood gradient-based training (PGT) with exact computation on mixture of Bernoulli data set. (b,c) QiDBM vs DBM on MNIST with CD.}
    \label{fig:combined}
\end{figure}

\subsection{\qlvms on Discrete Data}

To showcase the scalability and performance of \qlvms using \doem, we train Quantum interleaved Deep Boltzmann Machines (\ref{QiDBM}) using contrastive divergence to approximate the M-step gradient. We compare our method with \citet{taniguchi23}, the state of the art for training DBMs. Since we were unable to reproduce their reported results, we use results obtained from their official implementation\footnote{\url{https://github.com/iShohei220/unbiased_dbm}} with the hyperparameters recommended in their paper. \emph{No additional hyperparameter tuning} is performed on the quantum variant.

Following \citet{taniguchi23}, we use a \(\{0,1\}\) encoding instead of \(\{-1,+1\}\) to perform experiments on MNIST and Binarized MNIST \citep{deng2012}. The data sets contain 60,000 training and 10,000 testing images of size 28$\times$28. We evaluate the quality of generated images using the Fr\'{e}chet Inception Distance \citep[FID,][]{seitzer2020}, computed between 10,000 generated samples and the MNIST test set.

\begin{figure*}[t]
    \centering

    \begin{subfigure}[t]{0.47\textwidth}
        \includegraphics[width=\linewidth, trim=0 0 0 182, clip]{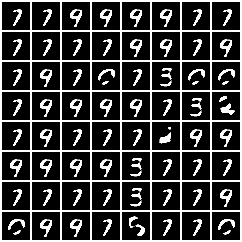}
        \caption{DBM at 175 epochs}
    \end{subfigure}
    \begin{subfigure}[t]{0.47\textwidth}
        \includegraphics[width=\linewidth, trim=0 0 0 182, clip]{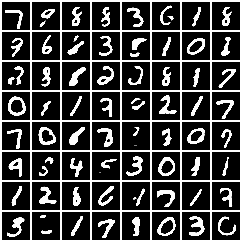}
        \caption{QiDBM at 175 epochs}
    \end{subfigure}
    \vspace{1em}

    \begin{subfigure}[t]{0.47\textwidth}
        \includegraphics[width=\linewidth, trim=0 0 0 182, clip]{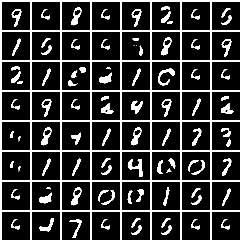}
        \caption{DBM at 206 epochs}
    \end{subfigure}
    \begin{subfigure}[t]{0.47\textwidth}
        \includegraphics[width=\linewidth, trim=0 0 0 182, clip]{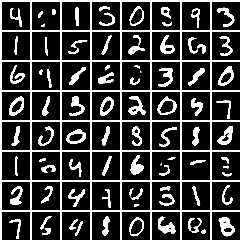}
        \caption{QiDBM at 206 epochs}
    \end{subfigure}
    \vspace{1em}
    
    \begin{subfigure}[t]{0.47\textwidth}
        \includegraphics[width=\linewidth, trim=0 0 0 182, clip]{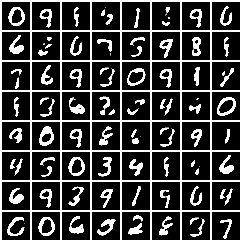}
        \caption{DBM at 337 epochs}
    \end{subfigure}
    \begin{subfigure}[t]{0.47\textwidth}
        \includegraphics[width=\linewidth, trim=0 0 0 182, clip]{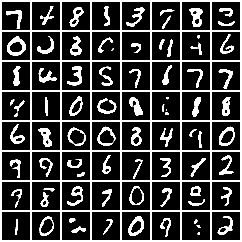}
        \caption{QiDBM at 337 epochs}
    \end{subfigure}
    \vspace{1em}
    
    \begin{subfigure}[t]{0.47\textwidth}
        \includegraphics[width=\linewidth, trim=0 0 0 182, clip]{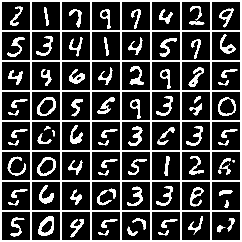}
        \caption{DBM at 679 Epochs}
    \end{subfigure}
    \begin{subfigure}[t]{0.47\textwidth}
        \includegraphics[width=\linewidth, trim=0 0 0 182, clip]{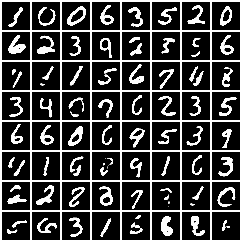}
        \caption{QiDBM at 679 Epochs}
    \end{subfigure}

   \caption{Generated samples during training: DBM (left)  QiDBM (right).}
    \label{fig:gen}
\end{figure*}
\subsubsection{MNIST}

To demonstrate the scalability of \qlvms on standard ML data sets under the \doem scheme, we train {\tt Qi-DBMs} on the MNIST dataset. Following \citet{taniguchi23}, each 8-bit image is mapped to 6272 binary visible units. We compare a three-layer \ref{QiDBM}---with quantum-bias terms in the middle layer and 6272 units per hidden layer---against two three-layer DBMs: one with 6272 units per hidden layer and another with 6273 units per hidden layer. The \ref{QiDBM} has 78.70M parameters while the DBMs have 78.69M and 78.71M parameters respectively. Training is carried out for 1000 epochs using stochastic gradient descent with Contrastive Divergence, a batch size of 600, and a learning rate of 0.001.

The FID values of generated samples, recorded every 10 epochs, are plotted in \cref{fig:scale}. The proposed method consistently outperforms both classical models of comparable size, achieving a \emph{45\% reduction in FID}. The \ref{QiDBM} \emph{converges within 400 epochs}, whereas both DBMs remain unstable beyond 500 epochs. The final FID scores are 62.77 for the \ref{QiDBM}, compared to 111.73 and 99.17 for the two DBMs.

In \cref{fig:gen}, we present qualitative comparisons of generated samples at randomly selected checkpoints during training. Notably, \ref{QiDBM} samples appear \emph{more coherent than those from DBMs} as early as 250 epochs. These findings show that scaling {\tt Qi-DBMs} to data sets such as MNIST is feasible and leads to \emph{substantial performance gains}.

\subsubsection{Binarized MNIST}

To further assess the impact of the additional quantum bias terms, we compare {\tt Qi-DBMs} and DBMs across different model sizes. \ref{QiDBM} and DBM models with two hidden layers are trained on 1-bit MNIST, with 490, 588, 686, and 784 units in each layer. Training is carried out for 100 epochs using stochastic gradient descent with Contrastive Divergence, a batch size of 600, and a learning rate of 0.001. Each setting is repeated over three random seeds to account for variability.

The minimum FID over 100 epochs against the number of parameters averaged over the three repeats is graphed in \cref{fig:perf}. Models with quantum bias consistently outperform their classical counterparts, achieving \emph{similar FID scores with significantly fewer parameters}. {\tt QiDBMs} outperforms the DBMs in all configurations achieving a minimum FID of 14.77 compared to the DBM’s 42.61, \emph{a 65\% improvement in generation quality}.

\subsection{\qlvms on Continuous Data}

To assess the applicability of the \doem algorithm to continuous data, we apply \cref{alg:cqdoem} and Contrastive Divergence to Quantum Gaussian-Bernoulli RBMs (\ref{QGRBM}). We compare our method with \citet{liao2023gaussianbernoulli}. We use results obtained from their official implementation\footnote{\url{https://github.com/lrjconan/GRBM}} with the hyperparameters recommended in their paper. \emph{No additional hyperparameter tuning} is performed on the quantum variant.

Following \citet{liao2023gaussianbernoulli}, we use a \(\{0,1\}\) encoding in the hidden units instead of \(\{-1,+1\}\) to perform experiments on Fashion MNIST \citep{xiao2017fashionmnist} and CelebA \citep{liu2015faceattributes}. The Fashion MNIST data set contains 60,000 training and 10,000 testing images of size 28$\times$28. The CelebA-32 data set contains 162,770 training and 19,962 testing images downsampled to size 32$\times$32. We assess the quality of the generated images using the Fréchet Inception Distance \citep[FID,][]{seitzer2020}, computed between generated samples and the test set with the same number of images generated as size of test set.

\begin{figure}[h]
    \centering
    \begin{subfigure}[t]{0.32\linewidth}
        \centering
        \includegraphics[width=\linewidth]{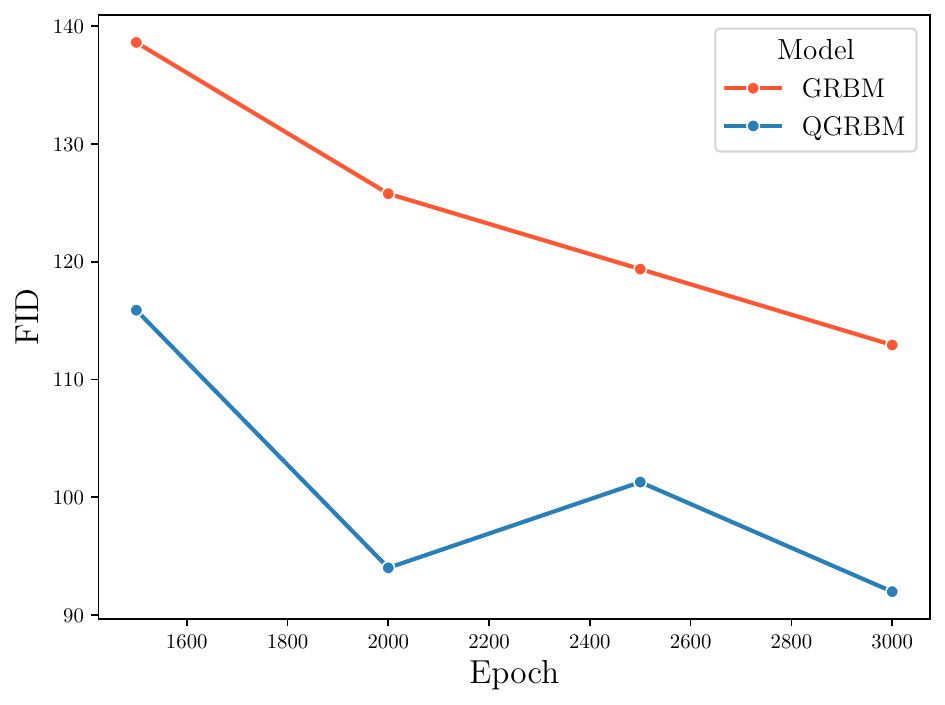}
        \caption{FID evolution of GRBM and QGRBM during training}
        \label{fig:fmnist_fid}
    \end{subfigure}
    \hfill
    \begin{subfigure}[t]{0.32\linewidth}
        \centering
        \includegraphics[width=\linewidth]{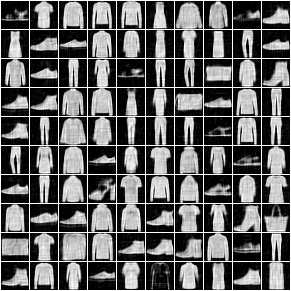}
        \caption{Samples generated by the QGRBM at Epoch 3000}
        \label{fig:fmnist_qgrbm}
    \end{subfigure}
    \hfill
    \begin{subfigure}[t]{0.32\linewidth}
        \centering
        \includegraphics[width=\linewidth]{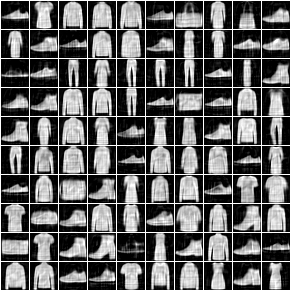}
        \caption{Samples generated by the GRBM at Epoch 3000}
        \label{fig:fmnist_grbm}
    \end{subfigure}
    \caption{Fashion MNIST}
    \label{fig:fmnist_combined}
\end{figure}

\subsubsection{Fashion MNIST}

We train a \ref{QGRBM} with 10,000 hidden units on the Fashion-MNIST dataset and compare it to a GRBM trained with Contrastive Divergence under the same conditions following \citet{liao2023gaussianbernoulli}. Training is carried out for 3000 epochs using stochastic gradient with Contrastive Divergence, a batch size of 512, and a learning rate of 0.01. The FID values of generated samples are plotted in \cref{fig:fmnist_fid}. The \ref{QGRBM} achieves \emph{a 20\% reduction in FID}. An illustrative comparison of generated samples is provided in \cref{fig:fmnist_qgrbm,fig:fmnist_grbm}, where we observe noticeably \emph{higher sample diversity} in \ref{QGRBM} relative to GRBM.

\begin{figure}[h]
    \centering
    \begin{subfigure}[t]{0.32\linewidth}
        \centering
        \includegraphics[width=\linewidth]{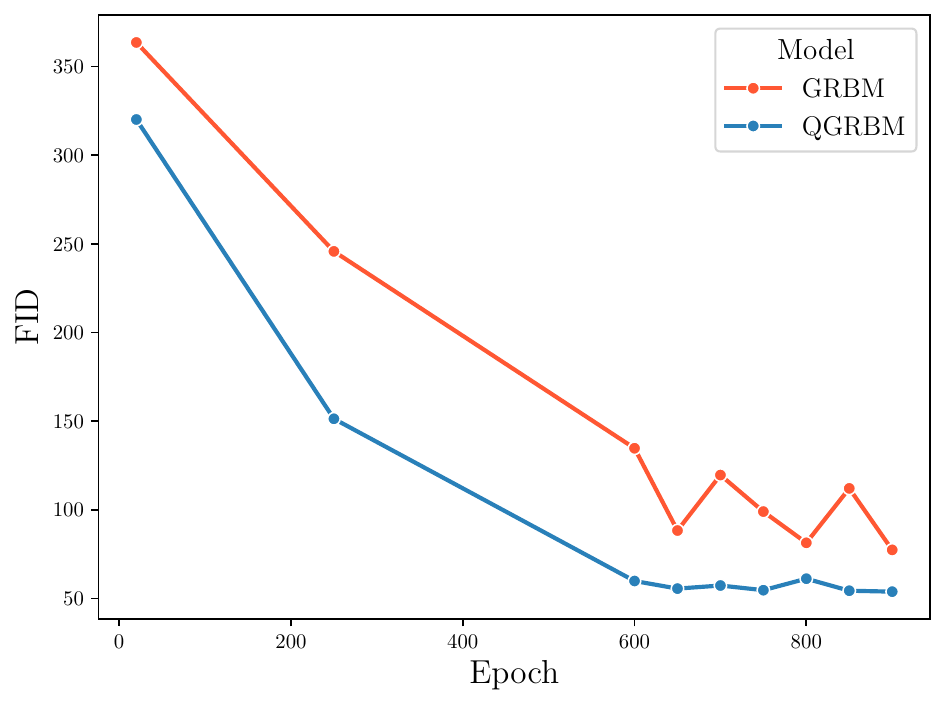}
        \caption{FID evolution of GRBM and QGRBM during training}
        \label{fig:celeba_fid}
    \end{subfigure}
    \hfill
    \begin{subfigure}[t]{0.32\linewidth}
        \centering
        \includegraphics[width=\linewidth]{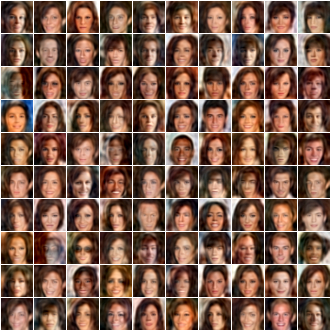}
        \caption{Samples generated by the QGRBM at Epoch 900}
        \label{fig:celeba_qgrbm}
    \end{subfigure}
    \hfill
    \begin{subfigure}[t]{0.32\linewidth}
        \centering
        \includegraphics[width=\linewidth]{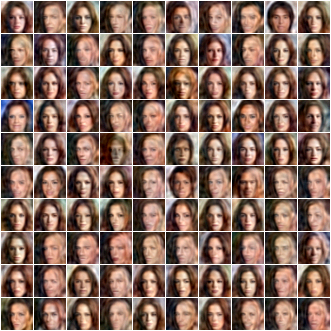}
        \caption{Samples generated by the GRBM at Epoch 900}
        \label{fig:celeba_grbm}
    \end{subfigure}
    \caption{CelebA-32}
    \label{fig:celeba_combined}
\end{figure}

\subsubsection{CelebA-32}

We train a \ref{QGRBM} on the Fashion-MNIST dataset and compare it to a GRBM trained with Contrastive Divergence under the same conditions, following \citet{liao2023gaussianbernoulli}.  with 10,000 hidden units for 3000 epochs. Training is carried out for 3000 epochs using stochastic gradient with Contrastive Divergence, a batch size of 512, and a learning rate of 0.01. The FID values of generated samples are plotted in \cref{fig:celeba_fid}. The \ref{QGRBM} achieves a \emph{30\% reduction in FID}. An illustrative comparison of generated samples is provided in \cref{fig:celeba_qgrbm,fig:celeba_grbm}, where we observe \emph{noticeably higher sample diversity} in \ref{QGRBM} relative to GRBM.

\subsection{Discussion}

The experiments provide clear answers to the three motivating questions. 
\begin{questions}[leftmargin=*,label=({Q}{{\arabic*}})]
    \item\textbf{Effectiveness of DO-EM.} The exact-likelihood studies show that \doem accelerates \qbm training by nearly two orders of magnitude relative to gradient-based optimization, while achieving comparable log-likelihoods. This demonstrates a tangible computational advantage in settings where exact inference is feasible.
    \item\textbf{Scalability of DO-EM.} The MNIST and Binarized MNIST results confirm that \qlvms scale reliably to models with more than 18,000 units. {\tt QiDBMs} trained under DO-EM remain stable throughout training and attain substantially lower FID scores—up to a 65\% improvement—indicating that the method handles standard ML data sets.
    \item\textbf{Performance of DO-EM.} experiments on both discrete and continuous domains show that \doem trained \qlvms outperform classical probabilistic LVMs of comparable size. {\tt QiDBMs} and {\tt QGRBMs} consistently achieve lower FID scores and exhibit higher sample diversity, reflecting a measurable improvement in generative quality.
\end{questions}
Taken together, these results demonstrate that DO-EM is computationally effective, scalable to standard ML data sets, and capable of delivering superior generative performance relative to their probabilistic counterparts.

%% file: Sections/8_conclusion.tex
\section{Conclusion}\label{sec:discussion}

This work advances density operator modeling by introducing \doem, an EM algorithm for density operator latent variable models that provides a formal guarantee of log-likelihood ascent.We further introduce \qlvms, a broad family of \dolvms uniting probabilistic visible variables with quantum latents, for which DO-EM provides a natural training procedure on standard ML datasets. Our experiments show that {\tt QiDBMs} and {\tt QGRBMs}---our newly introduced \qlvms---scale effectively to image datasets, providing the first image generation benchmark for a \qbm. Across these tasks, they consistently outperform comparable probabilistic models under similar computational budgets and hyperparameter settings, indicating that density-operator formulations can yield tangible performance gains. Finally, the close connection between \doem and quantum-information tools like the Petz recovery map positions it as a promising candidate for future quantum-hardware implementations, which we defer to subsequent work.

%% file: Appendix/1_analysis.tex
\section{Matrix Analysis}\label{app:analysis}
This appendix summarizes the essential matrix analysis tools used throughout the manuscript. For a more comprehensive treatment, see \citet{bhatia1997}.

We consider only finite-dimensional vector spaces of dimension $d$. The inner product between two elements of such a vector space is denoted by $\braket{u,v}$. Inner products are conjugate linear in the first variable and linear in the second variable. A vector space with an inner product is a finite-dimensional Hilbert space if $\braket{u,u}=0$ if and only if $u=0$. We denote the space of linear operators from a Hilbert space $\hil$ to itself by $\map(\hil)$. if the basis of $\hil$ is fixed, then every operator in $\map(\hil)$ has a unique matrix associated with it. 

An operator $A$ in $\map(\hil)$ is Hermitian if $A=A\h$. The spectrum of such an operator is denoted $\mathrm{Sp}(A)$. A Hermitian operator $A$ is positive semi-definite if $\braket{x,Ax}\geq 0$ for all $x$ in $\hil$. The operator is positive definite if $\braket{x,Ax}>0$ for all non-zero $x$ in $\hil$. According to the spectral theorem, a Hermitian operator $A$ can be expressed as
\begin{align*}
    A = \sum_{i=1}^d \lambda_i \Lambda(v_i)
\end{align*}
where $(\lambda_i,v_i)$ are eigenvalue-eigenvector pairs of $A$. A real valued function $f:\R\to\R$ can be naturally extended to act on operators via the spectral theorem, 
\begin{align*}
    f(A)\mathord{=}\sum_{i=1}^{d}f(\lambda_i)\Lambda(v_i).
\end{align*}
Several inequalities related to convex functions can be extended to such functions. We provide the finite-dimensional restatement of an inequality in \citet{dragomir2016}.
\begin{theorem}[Corollary 2.1 in \citeauthor{dragomir2016}, \citeyear{dragomir2016}]\label{thm:dragomir}
    Let $A$ be a Hermitian operator on a Hilbert space $\hil$ and assume that $\mathrm{Sp}(A)\subseteq[m,M]$ for some scalars $m$, $M$ with $m<M$. If $f$ is a continuous convex function on $[m,M]$ and $P$ be a non-zero positive semi-definite operator in $\map(\hil)$, then $m\leq \frac{\tr(PA)}{\tr(P)}\leq M$ and
    \begin{align*}
        f\left(\frac{\tr(PA)}{\tr(P)}\right)\leq \frac{\tr(Pf(A))}{\tr(P)}.
    \end{align*}
\end{theorem}
\subsection{Matrix exponential and logarithm}\label{app:expmlogm}
In this appendix, we specialize to the matrix exponential and logarithm functions, the most common functions used in this manuscript. The exponential function extended to all Hermitian operators is given by
\begin{align*}
    \exp(A)=\sum_{k=0}^\infty \frac{1}{k}A^k.
\end{align*}
The exponential of a Hermitian operator $A$ is positive definite. The inverse of the exponential function, the logarithm, can be defined on positive definite matrices as 
\begin{align*}
    \log (A) = - \sum_{k=1}^\infty \frac{1}{k}(I-A)^k.
\end{align*}

Let $X$ be in $\map(\hil_A)$ and $Y$ be in $\map(\hil_B)$. We state a few binary matrix operations and their relationship with the exponential function. The direct sum of $X$ and $Y$ is given by $X\oplus Y = \begin{pmatrix}
        X&0\\
        0& Y
    \end{pmatrix}$
    The exponential of $X\oplus Y$ is given by $\exp(X)\oplus \exp(Y)$. The tensor or Kronecker product is given by $X\otimes Y=\begin{pmatrix}
        X_{11}Y&\dots &X_{1d_A}Y\\
    \vdots&\vdots&\vdots\\
    X_{d_A1}Y&\dots&X_{d_Ad_A}Y
\end{pmatrix}$ where $X_{ij}$ are entries of the matrix $X$. The following lemma can be proved using the spectral theorem applied to the Kronecker product.
\begin{lemma}\label{thm:ksum}
    If $X$ is a Hermitian operator in $\map(\hil_A)$, $\exp(A)\otimes \id_B=\exp(A\otimes \id_B)$. Similarly, if $M$ is positive definite operator in $\map(\hil_A)$, $\log(M)\otimes \id_B=\log(M\otimes\id_B)$.
\end{lemma}

\subsection{M Step Gradient}\label{app:mstep}
In this appendix, we provide a proof of \cref{eq:gradient} presented in \citet{kappen2020}.
\begin{corollary}\label{thm:LTS}For a Hamiltonian-based model $\rho(\theta)=\exp(\H(\theta))/Z(\theta)$ with $\H(\theta)=\sum_r \theta_r\H_r$ and  and E step solution $\eta^{(t)}$, the gradient of $\mathcal{Q}(\theta;\theta\iter{\textrm{old}})$ in the M-step of \cref{alg:doem} with respect to the parameter $\theta_r$ is given by
    \begin{align*}
        \frac{\partial}{\partial \theta_r}\mathcal{Q}(\theta;\theta\iter{\textrm{old}})&=\braket{\H_r}_{\eta\iter{t}}-\braket{\H_r}_{\rho(\theta)}.
    \end{align*}
\end{corollary}
\begin{proof}
    The expression for $\mathcal{Q}(\theta;\theta\iter{\textrm{old}})$ is given by
    \begin{align*}
        \mathcal{Q}(\theta;\theta\iter{\textrm{old}})=\tr(\eta\iter{\textrm{old}}\log \rho(\theta))+S(\eta\iter{\textrm{old}})-S(\etav).
    \end{align*}
    We note that the latter two terms in the expression are not dependent on $\theta_r$. The logarithm of the model is given by $\log\rho(\theta)=\sum_r\theta_r\H_r-\log(Z(\theta))\id$. The gradients can then be written as
    \begin{align*}
        \frac{\partial}{\partial \theta_r}\tr(\eta\iter{\textrm{old}}\sum_r\theta_r\H_r)&=\tr(\eta\iter{\textrm{old}}\H_r),\text{ and }\\
        \frac{\partial}{\partial \theta_r}\tr(\eta\iter{\textrm{old}}\log Z(\theta)\id)&=\frac{\partial}{\partial \theta_r}\log \tr\exp(\H(\theta)).
    \end{align*}
    We know by the Lie-Trotter-Suzuki formula that $e^{\H(\theta)}=\lim_{t\to\infty}\left(e^{\H(\theta)/m}\right)^m$. Hence,
    \begin{align*}
        \frac{\partial}{\partial \theta_r}\log \tr\exp(\H(\theta))&=\frac{1}{\tr\exp(\H(\theta))}\tr\left(\int_{0}^{1}\mathrm{d}t e^{\H(\theta) t}\H_re^{\H(\theta)(1-t)}\right),\\
        &=\frac{1}{\tr\exp(\H(\theta))}\tr(e^{\H(\theta)}\H_r),\\
        &=\tr(\rho(\theta)\H_r).
    \end{align*}
    Collecting together the two partial derivatives, we obtain
    \begin{align*}
        \frac{\partial}{\partial \theta_r}\mathcal{Q}(\theta;\theta\iter{\textrm{old}})&=\tr(\eta\iter{t}\H_r)-\tr(\rho(\theta)\H_r)
    \end{align*}
    which is the required result.
\end{proof}
\subsection{Chain Rule for Classical-Quantum States}\label{app:crure}
In this appendix, we prove a chain rule for the Umegaki relative entropy of two classical-quantum states with respect to the same orthonormal basis. 
\begin{theorem}[Chain Rule for Classical-Quantum States]\label{thm:crure}
    Suppose $\eta$ and $\rho$ are two classical-quantum states on $\hil_A\otimes\hil_B$ with respect to the same orthonormal basis of $\hil_A$ given by
    \begin{align*}
        \eta = \sum_{i=1}^{d_A} q(\X=\x_i)\Lambda(\u_i)\otimes \etah(\x_i) \text{ and } \rho = \sum_{i=1}^{d_A} p(\X=\x_i)\Lambda(\u_i)\otimes \rhoh(\x_i).
    \end{align*}
    Then,
    \begin{align*}
        \du(\eta,\rho)=\kl{q(\X)}{p(\X)} + \sum_{\x_i} q(\X=\x_i) \du(\etah(\x_i),\rhoh(\x_i)).
    \end{align*}
\end{theorem}
\begin{proof}
    Rewrite the density operators as 
        \begin{align*}
        \eta = \sum_{i=1}^{d_A} \Lambda(\u_i)\otimes q(\X=\x_i)\etah(\x_i) \text{ and } \rho = \sum_{i=1}^{d_A} \Lambda(\u_i)\otimes p(\X=\x_i)\rhoh(\x_i).
    \end{align*}
    For any positive real $c$ and Hermitian operator $\O$, $\log(c\O)=\log(c)\id+\log(\O)$. Hence,
    \begin{align*}
        \tr(\eta\log \eta)&=\sum_{\x_i} q(\X=\x_i)\log q(\X=\x_i) + q(\X=\x_i)\log \etah(\x_i),\\
        \tr(\eta\log \rho)&=\sum_{\x_i} q(\X=\x_i)\log p(\X=\x_i) + q(\X=\x_i)\log \rhoh(\x_i).
    \end{align*}
    Combining the two expressions and using the definitions of the KL divergence and the Umegaki relative entropy, we obtain
    \begin{align*}
        \du(\eta,\rho)=\kl{q(\X)}{p(\X)} + \sum_{\x_i} q(\X=\x_i) \du(\etah(\x_i),\rhoh(\x_i)),
    \end{align*}
    which is the required expression.
\end{proof}
The chain rule provides a condition for when \cref{thm:mre} is saturated under the partial trace for classical-quantum states.

\subsection{Gibbs Sampling}\label{app:sampling}

In this appendix, we provide a justification for the Gibbs sampling step in \cref{eq:gibbs}. We first state a useful result. The Kronecker sum of two matrices $X$ in $\map(\hil_A)$ and $Y$ in $\map(\hil_B)$ is given by $X\otimes \id_B + \id_A \otimes Y$. Since $[X\otimes \id_B,\id_A \otimes Y]=0$, we have from Lemma \ref{thm:ksum} that
\begin{equation}\label{eq:kronsum}
\exp(X\otimes \id_B + \id_A\otimes Y)=\exp(X)\otimes \exp(Y). 
\end{equation}

Consider a Quantum RBM (\ref{QRBM}) with $m$ visible and $n$ hidden units given by the Hamiltonian
\begin{align*}
    \H(\theta) = -\sum_{i=1}^{m}\mathbf a_i\sz_i-\sum_{i=1}^{n}\mathbf b_i\sz_{m+i} -\sum_{i=1}^{m}\sum_{j=1}^{n}\mathbf W_{ij}\sz_i\sz_{m+j} -\sum_{i=1}^{n}\Gamma_i\sz_{m+i}.
\end{align*}
The model satisfies the \conds for classical data since it only has the transverse terms on the hidden units. For a given visible sample $\x$, the condition density operator for the hidden layer(well-defined due to the model being a \qlvm) is given by
\begin{align*}
    \H_{\hid}(\x,\theta) = -\sum_{j=1}^{n}\mathbf b_j^{\mathrm{eff}}\sz_j-\sum_{j=1}^{n}\Gamma_j\sx_j,
\end{align*}
\begin{equation*}
    \rhoh(\x|\theta)=\frac{1}{\mathcal Z_\hid(\x,\theta)}\exp(\H_{\hid}(\x,\theta))\quad\text{where}\quad\mathcal Z_{\hid}(\x,\theta)=\tr\exp(\H^{\hid}(\x,\theta)).
\end{equation*}
where $\mathbf b_j^{\mathrm{eff}}=\mathbf b_j+\sum_{i=1}^m \mathbf W_{ij}\x$. Using the definition of the Pauli Operators in \cref{eq:pauli}, we can rewrite this Hamiltonian as 
\begin{align*}
    \H_{\hid}(j|\x,\theta)&=-\mathbf b_j^{\mathrm{eff}}\sz-\Gamma_j\sx,\\
   \H_{\hid}(\x,\theta)&=\sum_{j=1}^n \id_{2^{j-1}}\otimes \H_{\hid}(j|\x,\theta)\otimes\id_{2^{n-i}}.
\end{align*}
Using \cref{eq:kronsum} recursively, the density operator $\rhoh(\x_i|\theta)$ is given by
\begin{align*}
    \rhoh(\x|\theta)=\bigotimes_{i=1}^n \frac{1}{\mathcal{Z}_\hid(j|\x,\theta)}\exp(\H_{\hid}(j|\x,\theta))\quad\text{where}\quad \mathcal{Z}_\hid(j|\x,\theta)=\tr\exp(\H_{\hid}(j|\x,\theta)).
\end{align*}
This gives us the desired sampling expression in \cref{eq:gibbs}.

%% file: Appendix/2_hardware.tex
\section{Compute Platform}
\label{app:hardware}

In this section, we provide additional details related to our experiments that we do not provide in the main body of the paper.

\paragraph{Implementation Details} We implement our proposed methods in PyTorch 2.5.1.

\paragraph{Hardware}
\cref{tab:M1,tab:M2} details the hardware we use to conduct our experiments.
Values in (*) indicated reported values obtained from \url{https://www.amd.com/en/products/accelerators/instinct/mi200/mi210.html}. Machine 1 runs Ubuntu 22.04.3 LTS with kernel 6.8.0-40-generic with the hardware in \cref{tab:M1}. 
Machine 2 runs Ubuntu 22.04.1 LTS with kernel 5.15.0-127-generic with the hardware in \cref{tab:M2}. 
Our software stack comprises of 3.12.8, PyTorch 2.5.1, torchvision version 0.20.1.

\begin{table}[H]
  \caption{Machine 1: Specifications of GPU hardware used for computation}
  \label{tab:M1}
  \centering
  \begin{tabular}{lll}
    \toprule
    CPU Model Name & AMD EPYC 9654 96-Core Processor      \\
    CPU(s)     & 192      \\
    Thread(s) per core & 1 \\
    Core(s) per socket     & 96       \\
    Socket(s)     & 2       \\
    NUMA node(s)     & 2       \\
    CPU MHz(Max)     & 3707.8120      \\
    L1d \& L1i cache & 6 MiB \\
    L2 cache & 192 MiB \\
    L3 cache & 768 MiB \\
    RAM & 1.48 TiB (DDR5, 4800 MT/s) \\
    GPU Model name & Instinct MI210\\
    GPU(s) & 4\\
    GPU Architecture & AMD Aldebaran\\
    Dedicated Memory Size(per GPU) &    64 GB\\
    ROCm Version & 6.0.2\\
    Peak FP32 Performance* &     22.6 TFLOPs\\
    Peak FP64 Performance*  & 22.6 TFLOPs\\
    Memory Clock* &     1.6 GHz\\
    Peak Memory Bandwidth* &     1.6 TB/s\\
    \bottomrule
  \end{tabular}
\end{table}

\begin{table}[H]
  \caption{Machine 2: Specifications of GPU hardware used for computation}
  \label{tab:M2}
  \centering
  \begin{tabular}{lll}
    \toprule
    CPU Model Name & Intel(R) Xeon(R) Silver 4216 CPU     \\
    CPU(s)     & 64      \\
    Thread(s) per core & 2 \\
    Core(s) per socket     & 16       \\
    Socket(s)     & 2       \\
    NUMA node(s)     & 2       \\
    CPU MHz(Max)     & 3200.0000      \\
    L1d \& L1i cache & 1 MiB \\
    L2 cache & 32 MiB \\
    L3 cache & 44 MiB \\
    RAM & 62 GiB \\
    GPU Model name & NVIDIA GeForce RTX 2080 Ti\\
    GPU(s) & 8\\
    Dedicated Memory Size(per GPU) &    11 GB\\
    CUDA Version & 12.9\\
    \bottomrule
  \end{tabular}
\end{table}